\documentclass[a4paper,UKenglish,cleveref, autoref, thm-restate]{lipics-v2021}
\hideLIPIcs  %uncomment to remove references to LIPIcs series (logo, DOI, ...), e.g. when preparing a pre-final version to be uploaded to arXiv or another public repository

\usepackage{algpseudocode}
\usepackage[ruled]{algorithm}
\usepackage{tikz}
\usepackage{multicol}
\usepackage{booktabs} 
\usepackage{svg}
\usepackage[normalem]{ulem}
\usepackage{arydshln}
\usepackage{graphicx}
\newcommand{\rsp}{\textsc{Robot Scheduling}~}
\usepackage{subcaption}

\newtheorem{problem}{Problem}

\newcommand{\pairs}{pairs}
\DeclareMathOperator{\dist}{dist}

\title{Fast and Near-Optimal Collision-Free Robot Scheduling On Paths} %TODO Please add

\author{Duncan Adamson}{Department of Computer Science, University of St Andrews, St Andrews, UK}{}{https://orcid.org/0000-0003-3343-2435}{}
\author{Nathan Flaherty}{Leverhulme Research Centre for Functional Materials Design, University of Liverpool, Liverpool, UK}{}{https://orcid.org/0000-0002-2798-4084}{}
\author{Igor Potapov}{School of Computer Science and Informatics, University of Liverpool, Liverpool, UK}{potapov@liverpool.ac.uk}{https://orcid.org/0000-0002-7192-7853}{}
\author{Paul G Spirakis}{School of Computer Science and Informatics, University of Liverpool, Liverpool, UK}{p.spirakis@liverpool.ac.uk}{https://orcid.org/0000-0001-5396-3749}{}
\author{Elena Zamaraeva}{Department of Computer Science, University of Mancester, Manchester, UK}{}{https://orcid.org/0000-0002-7948-2641}{}
\authorrunning{D. Adamson, N. Flaherty, I. Potapov, P.G. Spirakis, E. Zamaraeva} 
 %TODO mandatory; please add comma-separated list of keywords

\category{} %optional, e.g. invited paper

\relatedversion{} %optional, e.g. full version hosted on arXiv, HAL, or other respository/website
\nolinenumbers %uncomment to disable line numbering

\EventEditors{John Q. Open and Joan R. Access}
\EventNoEds{2}
\EventLongTitle{42nd Conference on Very Important Topics (CVIT 2016)}
\EventShortTitle{CVIT 2016}
\EventAcronym{CVIT}
\EventYear{2016}
\EventDate{December 24--27, 2016}
\EventLocation{Little Whinging, United Kingdom}
\EventLogo{}
\SeriesVolume{42}
\ArticleNo{23}
\begin{document}

\maketitle

\begin{abstract}
%In this paper, we consider the problem of scheduling a set of robots to complete tasks in a laboratory environment, modelled by a graph, while avoiding collisions. We compare four scheduling algorithms: a randomised algorithm, a fast greedy algorithm, an efficient but slower dynamic programming algorithm, and a slow but optimal integer programming approach. Our results show that for a line, the dynamic programming algorithm, while slower than the randomised and greedy, returns a near-optimal solution in most cases at a fraction of the time required for the integer programming algorithm. On the other hand, we find that when considering lattices, the performance of the algorithms rapidly decreases as the size of the lattice increases. This work then serves two purposes. First, we show that the dynamic programming algorithm is, in general, close to optimal for lines while operating at a fraction of the time. Second, we provide a tool for finding an optimal solution in the form of an integer programming formulation.

In this paper, we address the problem of scheduling a set of robots to complete tasks in a laboratory environment, modelled as a graph, while avoiding collisions. We analyze the dynamic programming algorithm (PA) introduced in \cite{ourPreviousPaper} and present three baselines for comparison: an integer programming approach (IP) that always yields an optimal solution, a greedy algorithm (GA), and a simple randomized algorithm (RA). We show that for a path graph, PA, GA, and RA find solutions several orders of magnitude faster than IP (the optimal baseline), with PA returning optimal results in the vast majority of cases. Our scaled experiments comparing non-optimal algorithms show that the average schedule timespan produced by PA is less than half that of RA and GA. This outperformance is consistent across varying path lengths, task durations and distributions, number and allocations of tasks and robots, and task-to-robot ratios. This work serves two purposes. First, we present three algorithms for scheduling on line graphs, including a novel integer programming formulation for finding optimal solutions. Second, we demonstrate that PA produces near-optimal schedules that outperform all non-optimal baselines while maintaining a comparable runtime. Code is available at \url{https://github.com/sea26-robots/code}.
\end{abstract}

%%% Use this command to specify a few keywords describing your work.
%%% Keywords should be separated by commas.


%%%%%%%%%%%%%%%%%%%%%%%%%%%%%%%%%%%%%%%%%%%%%%%%%%%%%%%%%%%%%%%%%%%%%%%%

%%% Include any author-defined commands here.
         
\newcommand{\BibTeX}{\rm B\kern-.05em{\sc i\kern-.025em b}\kern-.08em\TeX}

%%%%%%%%%%%%%%%%%%%%%%%%%%%%%%%%%%%%%%%%%%%%%%%%%%%%%%%%%%%%%%%%%%%%%%%%

%%% The following commands remove the headers in your paper. For final 
%%% papers, these will be inserted during the pagination process.

%%% The next command prints the information defined in the preamble.

\maketitle 

%%%%%%%%%%%%%%%%%%%%%%%%%%%%%%%%%%%%%%%%%%%%%%%%%%%%%%%%%%%%%%%%%%%%%%%%

\section{Introduction}

Recent years have seen an increase in the use of moving autonomous robots within a range of environments, from manufacturing \cite{Liu2023} to Unmanned Aerial Vehicles \cite{qamar2023trmaxalloc}. We make particular note of our primary motivation, that being the scheduling of robots within chemistry labs and the corresponding significant and expanding body of work concerning robotic chemists. Initial work on these systems focused on building robots performing reactions within fixed environments \cite{granda2018controlling,king2011rise,langner2019ternary,li2015synthesis,macleod2020selfdriving}, however recently Burger et al. \cite{burger2020mobile} have presented a robot capable of moving within a laboratory and completing tasks throughout the space. The works of Burger et al. \cite{burger2020mobile} and Liu et al. \cite{Liu2023} provide the main motivation for this work, namely the \emph{Robot Trajectory Planning Problem}
%of moving robots 
within a laboratory environment (as presented by Burger et al. \cite{burger2020mobile}) while avoiding collisions (as investigated in the manufacturing context by Liu et al. \cite{Liu2023}).\looseness=-1

Throughout all such applications, there are two %clear
objectives for robots, minimising the time of completing all tasks and completing tasks safely by avoiding collisions.
% (1) minimise the time of completing all tasks and
% (2)  complete tasks safely avoiding collisions.
%
%
%
%the robots need to complete their tasks quickly and (2) the robots must complete their tasks safely. 
%
Under this definition, our problem may be viewed as a variant of the well-known travelling salesman problem.
% This seemingly simple problem of finding the best path for a robot to complete all the tasks is a variant of the well-known travelling salesman problem.
%
%Neither of these problems %are trivial, indeed, the %seemingly simple challenge %of finding the best path %for a robot to take in %order to complete all its %task is an example of the %well known travelling %salesman problem. 
In the other direction, scheduling the robots in such a way that they avoid collision is known to be NP-hard \cite{ourPreviousPaper}, even when the space is simplified to a simple discrete structure such as a tree.\looseness=-1

\textbf{Our Contribution.}
We address the problem of scheduling a set of robots to complete tasks in a laboratory environment, modelled as a graph , while avoiding collisions. 
Our primary experimental result is showing that the algorithm presented in \cite{ourPreviousPaper} is not only optimal in the restricted case studied there, but it is also near-optimal in several other cases, in particular for paths where the duration of tasks varies.
%
% We first address a motivating natural sub-problem in general graphs, namely does there exist a path connecting all tasks and if so can we find one? We do this because a good approximation algorithm is known \cite{ourPreviousPaper} for scheduling the robots in a path. We provide theoretical results about the complexity of this path problem. 
We focus on scheduling the robots when all tasks and all robots are on a given path, motivated by the hardness of this problem for general graphs. We first formulate an integer programming model (IP) for this problem and prove its correctness. We then present three scheduling algorithms and compare them with the Partition Algorithm (PA) from \cite{ourPreviousPaper}.
% We show that PA provides a near-optimal collision-free scheduling for the fraction of time required by the algorithm with guaranteed optimality.
We demonstrate that PA outperforms non-optimal baselines, providing schedules of less than half the timespan on average and confirm its excellent scalability with respect to various parameters, including the path length, the number of tasks, and the duration of the tasks. We show that the PA's performance remains superior compared to the baseline under various conditions, including various tasks and robots' position distributions and fixed task-to-robot ratios.
Finally, we show that even when we restrict ourselves to grid graphs, the scheduling problem remains computationally hard, even with uniform task durations and only a single robot, as does the auxiliary problem of finding a path between all tasks.\looseness=-1
% vaallowing us to find a fast, collision-free, schedule for a set of robots on a path graph. We experimentally verify the optimality of  the algorithm for finding a collision-free  schedule on paths when the tasks that each robot must complete all have the same length with the optimal solution. Additionally, we provide an integer linear program that solves this problem optimally on any input graph, at the expense of far more computational resources being required. By comparing the solutions of these two algorithms, including the time required by the schedule itself, and the run time of each algorithm, we show that the heuristic algorithm is optimal or near optimal in nearly all cases, with a far faster run time than the integer linear program.

\textbf{Related Work.}
% NEW STUFF BEGIN
Research on collision-free robot coordination is closely related to the Multi-Agent Pathfinding (MAPF) problem \cite{SHARON201540,standley2010finding,Stern_2021}, where a set of agents must move from given start locations to assigned goal locations without collisions. Anonymized MAPF removes fixed agent–goal assignments, requiring only that each agent reaches any goal from a specified set; however, it retains the one-agent--one-goal requirement. Building on anonymized MAPF, the Multi-Agent Pickup-and-Delivery problem \cite{liu2019task} extends this formulation by introducing tasks composed of pickup and delivery locations, coupling path planning with task assignment and sequencing. 

The collision-free robot scheduling formulation proposed in \cite{ourPreviousPaper} and studied in this paper further extends anonymized MAPF by generalizing a visit to a MAPF target into a task-solving process defined by a required number of sequential time steps spent at the task location. This formulation also removes the one-robot--one-goal limitation typical of MAPF. In \cite{ourPreviousPaper}, the authors proposed PA for the collision-free robot scheduling and showed that, when every task is of uniform duration, PA can find optimal schedules for line, cycle, and tadpole graphs. On the other hand, they showed that, even for uniform duration tasks, this problem is NP-hard in most cases, including trees with multiple agents, or planar graphs with only one agent. Looking more broadly at exploration problems, we can consider the work on temporal graph exploration, both in general \cite{arrighi2023kernelizing,erlebach2021temporal,erlebach_et_al:LIPIcs.ICALP.2019.141,erlebach2022parameterized,michail2016traveling}, and for specific graph classes \cite{adamson2022faster,akrida2021temporal,bodlaender2019exploring,bumpus2023edge,deligkas2022optimizing,erlebach2018faster,erlebach2022exploration,taghian2020exploring}. We note that the structure of temporal graphs is close to the challenges implemented in our graph by agents blocking potential moves from each other, making this work of interest to us.
Particularly relevant to us is the work of Michail and Spirakis \cite{michail2016traveling}, who showed that the problem of determining the fastest exploration of a temporal graph is NP-hard, and, furthermore, that there is no constant factor approximation algorithm for the shortest exploration (in terms of the length of the path found by the algorithm, compared to the shortest path exploring the graph) unless $P = NP$.
Looking at our Integer Programming baseline approach, we note that the travelling salesman problem was one of the first and most heavily studied applications of integer programming \cite{dantzig2016linear}, making this a natural tool for our project.\looseness=-1

% Looking 

% In terms of positive results, the work of Erlebach et al. \cite{erlebach2021temporal} provided a substantial set of results that have formed the basis for much of the subsequent work on algorithmic results for temporal graph exploration. Of particular interest to us are the results that show that, for temporal graphs that are connected in every timestep, an agent can visit any subset of $m$ vertices in at most $O(n m)$ time, and provide constructions for faster explorations of graphs with $b$ agents and an $(r, b)$-division ($O(n^2 b / r + n r b^2)$ time), and $2 \times n$ grids with $4 \log n$ agents ($O(n \log n)$ time).

\section{Preliminaries}\label{sec:prelim}
% Before proceeding we need to introduce some necessary definitions, here we follow the same conventions to \cite{ourPreviousPaper}.
A (simple) graph $G=(V,E)$ is a pair consisting of a set of vertices $V$ and a set of edges $E \subseteq V \times V$. A \emph{walk} in a graph $G$ of length $\ell$ is a sequence of $\ell$ edges of the form $(v_1, v_2), (v_2, v_3), \dots, (v_{\ell - 1}, v_{\ell})$. Any walk $w$ can visit the same vertex multiple times and may use the same edge multiple times. 
A walk without any such repetitions is called a \emph{path}. A path which visits all vertices in a graph is called a \emph{Hamiltonian path}.
% We define a $n \times m $ \emph{lattice} graph to be a square-grid with $n$ vertices in each row and $m$ vertices in each column.
Given a walk $w = (v_1, v_2), (v_2, v_3), \dots, (v_{\ell - 1}, v_{\ell})$, we denote by $\vert w \vert$ the total number of edges in $w$, and by $w[i]$ the $i^{th}$ edge in $w$. The \emph{distance} between a pair of vertices $v_1, v_2$, denoted $\dist(v_1, v_2)$, is the number of edges in the shortest walk starting at $v_1$ and ending at $v_2$. In this paper, we also allow walks to contain self-adjacent moves, i.e. moves of the form $(v_i, v_i)$ for every vertex in the graph. 
In this problem, we consider a set of agents, which we call \emph{robots}, moving on a given graph $G = (V, E)$ and completing a set of tasks $\mathcal{T} = \{t_1, t_2, \dots, t_m\}$. Each robot has some initial position $sv \in V$ and each takes a single timestep to traverse an edge.
A \emph{task} $t_i$ is a pair $(v_i,d_i)$ where $v_i$ is the vertex the task is located and $d_i$ is the duration of that task (i.e. the number of timesteps a robot needs to be on that task's vertex for in order to complete it).
A \emph{schedule} $C$ for a robot is a sequence alternating between walks and tasks (which could begin and end with any combination of walk and task, task and task, walk and walk, or task and walk) starting at the initial vertex for its robot.
The \emph{timespan} (or makespan) of a schedule is the number of timesteps that schedule takes to be executed, formally if a schedule $C$ consists of tasks $t_1,...,t_m$ and walks $w_1,..,w_p$ then the timespan of $C$ (denoted $\vert C\vert$) is given by $\vert C \vert = \left( \sum_{i \in [1, p]} \vert w_i \vert \right) + \left(\sum_{j \in [1, m]} d_j  \right)$.
The \emph{walk representation} $\mathcal{W}(C)$ of a schedule $C$ is an ordered sequence of edges formed by replacing the task $t_i = (v_i,d_i)$ in $C$ with a walk of length $\vert t_i \vert = d_i$ consisting only of the edge $(v_{i}, v_i)$, then concatenate the walks together in order. 
A set of schedules $\mathcal{C} = C_1,...,C_k$ is said to be task-completing if for each task $t \in \mathcal{T}$ there exists a schedule $C_i$ such that $t \in C_i$.
We call such a set of schedules $\mathcal{C}$ \emph{collision-free} if there is no timestep where the robots traverse the same edge or inhabit the same vertex. In other words, for every pair of schedules $C_i,C_j$ where $i \neq j$ and all timesteps $s \in [\vert C_i \vert ]$, $\mathcal{W}(C_i)[s] = (v, u)$ and $\mathcal{W}(C_j)[s] = (v', u')$ satisfies $u \neq u'$, $v \neq v'$ and $(v, u) \neq (u', v')$.
Given 2 sets of schedules $\mathcal{C}$ and $\mathcal{C}'$, we say $\mathcal{C}$ is \emph{faster} than $\mathcal{C'}$ if $\max_{C_i \in \mathcal{C}} \vert C_i \vert < \max_{C_j' \in \mathcal{C}'} \vert C_j' \vert$.

Given a graph $G = (V, E)$, set of $k$ robots $R_1, R_2, \dots, R_k$ starting on vertices $sv_1, sv_2, \dots$, $sv_k$, and set of tasks $\mathcal{T}$, a \emph{fastest} task-completing, collision-free set of $k$-schedules is the set of schedules $\mathcal{C}$ such that any other set of task-completing, collision-free schedules is no faster than $\mathcal{C}$. Note that there may be multiple such schedules.\looseness=-1
%The aim of the $k$-\textsc{Robot Scheduling} problem  is to find a shortest collision-free and task-completing set of schedules for the $k$ robots.
\begin{problem}[$k$-\rsp]
    \label{prob:robot-scheduling}
    Given a graph $G = (V, E)$, set of $k$  robots $R_1, R_2, \dots, R_k$ starting on vertices $sv_1, sv_2, \dots, sv_k$, and set of tasks $\mathcal{T}$, what is the fastest task-completing, collision-free set of $k$-schedules $\mathcal{C} = (C_1, C_2, \dots, C_k)$ such that $C_i$ can be assigned to $R_i$, for all $ i \in [1, k]$?\looseness=-1
\end{problem}

\subsection{$k$-Partition Algorithm}%(partition on Line/Cycle/Trees)}

% REWRITE TO MORE CLEARLY EXPLAIN PA

In this work, we compare the previous \emph{$k$-Partition Algorithm} by Adamson et al. \cite{ourPreviousPaper} for \rsp on path graphs with a set of new algorithmic tools, in particular an integer linear programming formulation guaranteeing optimal solutions. We provide a brief overview of the operation of this algorithm, with full pseudocode provided in Algorithm \ref{alg:pa} in Appendix \ref{app:implementation}.\looseness=-1

At a high level, $k$-partition works by partitioning the path graph into a set of $k$-contiguous, though not necessarily disjoint, sub-instances, each containing a single robot, and some subset of the tasks, with the union of these instance corresponding to the original input. We assume the path is on the vertices $v_1, v_2, \dots, v_n$, with $v_i$ \emph{left} of $v_j$ if $i < j$, and \emph{right} if $j > i$. The algorithm works using a dynamic programming approach, starting with working out the best way to solve this instance for a single robot, by convention the left-most, then using this solution to compute the solution for the two left-most robots and so on.\looseness=-1

More explicitly, we use $SC_{k}(P, \mathcal{T}, \mathcal{S})$ to denote the length of the solution given by the partition algorithm for the instance of $k$-\rsp on the path, $P$, with set of tasks $\mathcal{T} = \{t_1, t_2, \dots, t_m \}$ and set of starting positions $\mathcal{S} = \{s_1, s_2, \dots, s_k\}$. We assume, without loss of generality, that $t_i$ is located left of $t_{i + 1}$, and $s_i$ is left of $s_{i + 1}$. First, for every $i \in [m]$, we compute the value of $SC_{1}(P, \{t_1, t_2, \dots, t_i\}, \{s_1\})$ via the helper function $C(P,\mathcal{T}', s)$, computing the optimal solution for one robot to solve the instance of $1$-\rsp on the graph $P$, set of tasks $\mathcal{T}' = \{(p_1, d_1), (p_2, d_2), \dots, (p_{m'}, d_{m'}) \}$, and starting vertex $v_i$, with $C(P, \mathcal{T}', v_i) = \min(\dist(v_i, p_1), \dist(v_i, p_{m'})) + \dist(p_1, p_{m'}) + \sum_{i \in [m']} d_i$. Therefore, $SC_{1}(P, \{t_1, t_2, \dots, t_i\}, \{s_1\}) = C(P, \{t_1, t_2, \dots, t_i\}, \{s_1\})$.\looseness=-1

In the general case, the value of $SV_{k'}(P, \{t_1, t_2, \dots, t_i\}, \{s_1, s_2, \dots, s_{k'}\})$ is computed from the values of $SV_{k' - 1}(P, \{t_1, t_2, \dots, t_{i'}\}, \{s_1, s_2, \dots, s_{k' - 1}\})$, for every $i' \in [i]$. This is determined by splinting the tasks in $\{t_1, t_2, \dots, t_{i}\}$ between the $k'^{th}$ robot, and the previous $k - 1$ robots, using the known solutions to each sub-instance as a basis. Formally, 
\[SV_{k'}(P, \{t_1, t_2, \dots, t_i\}, \{s_1, s_2, \dots, s_{k'}\}) =\]\[\min_{i' \in [0,i]} \max\left(SV_{k' - 1}(P, \{t_1, t_2, \dots, t_{i'}\}, \{s_1, s_2, \dots, s_{k' - 1}\}), C(P, \{t_{i' + 1}, t_{i' + 2}, \dots, t_{i} \}, s_{k'}) \right),\]
corresponding to the slower of the time required for the robots $R_1, R_2, \dots R_{k' - 1}$ to complete tasks $\{t_1, t_2, \dots, t_{i'}\}$, and the time for robot $R_{k'}$ to complete the tasks $\{t_{i' + 1}, t_{i' + 2}, \dots, t_{i}\}$, for some $i' \in [0, i]$.
A full proof, including the collision-free nature of the corresponding schedule, and the optimality of this algorithm for tasks of equal duration can be found in \cite{ourPreviousPaper}.

\section{Integer Linear Programming Formulation of the Collision-Free Robot Scheduling Problem on a Path}
\label{sec:integer_program}

We introduce our integer linear programming model (IP) aimed to find a fastest, collision-free set of schedules. First, we outline the variables used in the program, before moving on to the constraints and objective function. We prove the correctness of the program as the corresponding constraints are presented.
Note that our IP formulation uses a polynomial number of variables and constraints relative to the size of the input graph, the number of robots and the number of tasks.
The primary variables used in this program to represent the schedule of each robot are the set of binary variables $x_{r, v, t} \in \{0, 1\}$, where $r \in [k], v \in V$ and $t \in [\tau]$. Where $\tau$ is an upper bound on the time a schedule could take, we use $\tau = n\cdot k + \sum_t d_t $. Here $x_{r, v, t} = 1$ iff the $r^{th}$ robot is on vertex $v$ at timestep $t$. We use these variables to construct a schedule for each robot by setting robot $r$ to be at vertex $v$ at timestep $t$ iff $x_{r, v, t} = 1$. If the robot $r$ occupies some vertex $v$ for containing a task requiring $\ell$ timesteps for $\ell$ consecutive timesteps, we assume that $r$ completes the task while at that vertex.
Additionally, we introduce the binary variables $TC_{i, t} \in \{0, 1\}$, for every $i \in [T], t \in [\tau]$ with $TC_{i, t} = 1$ iff the $i^{th}$ task has been completed by timestep $t$. Finally, we use the set of binary variables $AC_{t} \in \{0, 1\}$, for every $t \in [\tau]$ to denote if every task in the graph has been completed by timestep $t$.\looseness=-1

We present our constraints. First, we introduce the \emph{movement constraints}, i.e. the constraints ensuring that the robot is at exactly one position at each timestep, that the robot can only stay in place, or move between adjacent vertices, between sequential timesteps, and that the schedules are collision-free.
\begin{align}
    \sum\limits_{v \in V} x_{r, v, t} = 1,  \forall r \in [k], t \in [\tau]\label{eq:only_one_robot_position_per_timestep}\\
    x_{r, v, t} \leq \sum\limits_{v' \in N(v) \cup \{ v \}} x_{r, v', t - 1},   \forall r \in [k], v \in V, t \in [2, \tau] \label{eq:adjancent_robots}\\
    \sum\limits_{r \in [k]} x_{r, v, t} \leq 1,   \forall v \in V, t \in [\tau]\label{eq:only_one_robot_per_vertex_timestep}\\
    x_{r, v, t} + x_{r, v', t - 1} + x_{r', v, t - 1} + x_{r', v', t} \leq 3, \notag \\
   \forall r \in [k], r' \in [k] \setminus \{ r \}, (v, v') \in E, t \in [\tau]  {}\label{eq:no_crossover}
\end{align}

\begin{lemma}
    \label{lem:collsion_free}
    Constraints \ref{eq:only_one_robot_position_per_timestep}, \ref{eq:adjancent_robots}, \ref{eq:only_one_robot_per_vertex_timestep}, and \ref{eq:no_crossover} ensure that the schedule given by the integer program is collision-free.
\end{lemma}

\begin{proof}
     Constraint (\ref{eq:only_one_robot_position_per_timestep}) ensures that each robot occupies exactly one position during each timestep, as otherwise $\sum\limits_{v \in V} x_{r, v, t} > 1$, if $r$ occupies two positions, or, $\sum\limits_{v \in V} x_{r, v, t} = 0$, if $r$ is not assigned any position.  Constraint (\ref{eq:adjancent_robots})  ensures that robot $r$ can be at vertex $v$ at timestep $t$ iff $r$ was at $v$, or some vertex neighbouring $v$, in the previous timestep, for any timestep in $[2, \tau]$, i.e. any timestep other than the first one. Next Constraint (\ref{eq:only_one_robot_per_vertex_timestep}) guarantees that no vertex can be occupied by more than one robot at any given timestep. Finally, Constraint (\ref{eq:no_crossover}) ensures that, given any pair of robots $r, r' \in [k]$ where $r \neq r'$, we ensure that $r$ and $r'$ do not attempt to cross the same edge in the same timestep. Explicitly, if $ x_{r, v, t} + x_{r, v', t - 1} + x_{r', v, t - 1} + x_{r', v', t} = 4$, then $r$ moves from $v'$ to $v$ in timestep $t$, and $r'$ moves from $v$ to $v'$ in timestep $t$. Otherwise, one or both of $r$ and $r'$ do not use $(v, v')$ in timestep $t$.
    Therefore, the schedule is collision-free.\looseness=-1
\end{proof}

We present our constraints regarding the completion of tasks.

\begin{align}
    TC_{i, t} \leq TC_{i, t - 1} + \max_{r \in [k]} \left(\sum\limits_{j \in [t - d_i, t]} x_{r, v_i, j} / d_i\right) & & \forall i \in [T], t \in [\tau]\label{eq:note_when_tasks_are_complete}\\
    TC_{i, t} \geq TC_{i, t - 1} & & \forall i \in [T], t \in [\tau] \label{eq:tasks_stay_completed}
\end{align}

%{\color{blue} To late to really update, but we could use
%\begin{align*}
%        TC_{i, t} \geq \left(\sum\limits_{j \in [t - d_i, t]} x_{r, v_i, j} / (d_i - 1)\right) - 1 & & \forall r \in [k], i \in [T], t \in [\tau]
%\end{align*}
%As a simplification of constraint \ref{eq:note_when_tasks_are_complete} that also forces $TC_{i, t}$ to be one as soon as we complete the task.
%}

\begin{lemma}
    \label{lem:TC_constraints_work}
    The value of $TC_{i, t}$ is one only if some robot has completed task $i$ by timestep $t$.
\end{lemma}

\begin{proof}
    From Constraint (\ref{eq:note_when_tasks_are_complete}) it follows that the value of $TC_{i, t}$ is one only if either $TC_{i, t - 1} = 1$ or $\sum\limits_{j \in [t - d_i, t]} x_{r, v_i, j} / d_i = 1$, for some robot $r$. Note that, by our construction, we assume that if robot $r$ remains on the vertex $v_i$ for $d_i$ timesteps, then it will complete task $i$, therefore, if $\sum\limits_{j \in [t - d_i, t]} x_{r, v_i, j} / d_i = 1$ for any robot $r \in [k]$, the task $i$ must be completed. Further, by Constraint (\ref{eq:tasks_stay_completed}), if $TC_{i, t - 1} = 1$, indicating that the task has been completed, then $TC_{i, t} = 1$.\looseness=-1
    % Therefore, the correctness follows.
\end{proof}

\begin{lemma}
    \label{lem:TC_constaints_allow_the_right_answer}
    Given any schedule where task $i$ is completed by timestep $t$, the value of $TC_{i,t}$ may be set to one.
\end{lemma}

\begin{proof}
    Observe that task $i$ can be completed at timestep $t$ if there exists some robot $r$ and timestep $t' \in [t]$ such that $x_{r, i j} = 1$, for every $j \in [t' - d_i, t']$. Therefore, by Constraint (\ref{eq:note_when_tasks_are_complete}), $TC_{i, t'} = 1$. By extension, if $TC_{i, t'} = 1$ then, for any $t'' \in [t' + 1, \tau]$, we can set $TC_{i, t''}$ to $1$.\looseness=-1
    % Thus we get the statement.
\end{proof}

Finally, we provide our constraint for the all complete $AC_t$ variable.

\begin{align}
    AC_t \leq \sum\limits_{i \in [\vert T \vert]} TC_{i, t} / \vert T\vert & & \forall t \in [\tau]\label{eq:all_complete}\\
    \sum_{t \in [\tau]} AC_t = 1 \label{eq:only_one_AC}
\end{align}

\begin{lemma}
    \label{lem:AC}
    The value of $AC_t$ is $1$ only if every task has been completed by timestep $t$.
\end{lemma}

\begin{proof}
    From Constraint (\ref{eq:all_complete}), $AC_t$ can be $1$ if and only if $TC_{i, t} = 1$, $\forall i \in [T]$. From Lemma \ref{lem:TC_constaints_allow_the_right_answer}, the value of $TC_{i, t}$ is one only if the task $i$ has been complete by timestep $t$.
    Therefore, $AC_t$ is one only if all tasks have been complete.
    % From Lemma \ref{lem:C_constaints_allow_the_right_answer}, given any timestep $t' \in [T]$ where task $i$ has been completed, there is an assignment of variables such that $TC_{i, t} = 1$. The
\end{proof}

\begin{theorem}
    The integer program formulated with Constraints (\ref{eq:only_one_robot_position_per_timestep}), (\ref{eq:adjancent_robots}), (\ref{eq:only_one_robot_per_vertex_timestep}),
(\ref{eq:no_crossover}),
(\ref{eq:note_when_tasks_are_complete}),
(\ref{eq:tasks_stay_completed}),
(\ref{eq:all_complete}) and the objective $$\textrm{Minimise} \sum_{t \in [\tau]}t\cdot AC_t$$% $$ \textrm{Maximise } \sum_{t\in [\tau]} AC_t$$
correctly finds the fastest, collision-free, task completing schedule for a given instance of $k$
\end{theorem}

\begin{proof}
    From Lemmas \ref{lem:collsion_free}, \ref{lem:TC_constraints_work}, 
\ref{lem:TC_constaints_allow_the_right_answer}, and \ref{lem:AC} we have that the program computes a valid collision-free schedule, with $AC_t = 1$ only if every task has been completed by timestep $t$. %Now, observe that the optimal solution to this program is the solution such that $AC_t = 1$ for the maximum number of values of $t \in [T]$. 
Now by Constraint \ref{eq:only_one_AC} exactly one value of $AC_t$ can equal 1, and by the objective function we seek to minimise the value of $t$ such that $AC_t =1$. 
Following Lemma \ref{lem:TC_constaints_allow_the_right_answer}, for any given schedule such that task $i$ is completed by timestep $t$, there is a valid assignment of variables such that $TC_{t, i} = 1$. Therefore, given any schedule such that every task is complete by timestep $t$, there is an assignment of variables such that $AC_t = 1$. 
Thus, by finding the assignment such that $\sum_{t \in [\tau]}t  \cdot AC_t$ is minimised, we find the fastest collision-free schedule that completes all the tasks in the minimum number of timesteps, giving the proof.
\end{proof}

% The following lemma hints at how the list of constraints in IP can be further reduced and replaced by the collision fix of the result.

\begin{lemma}\label{lem:traverse}
Let $\mathcal{C}$ be a non-collision-free set of schedules $C_1,\dots,C_k$ of robots $R_1,\dots,R_k$ such that there is no timestep in which the robots inhabit the same vertex. Then there exists a collision-free task-completing set of schedules $\mathcal{C}'$ of the same timespan as the timespan of $\mathcal{C}$.
\end{lemma}
\begin{proof}
If $\mathcal{C}$ is a non-collision-free set of schedules and none of the robots inhabit the same vertex at the same timestep, then there exists a pair of robots $R_j, R_{j+1}$ and an edge $(v_i,v_{i+1})$ such that $R_j$ and $R_{j+1}$ traverse $(v_i,v_{i+1})$ at some timestep $t$. Let $C_j(1:t)$, $C_{j+1}(1:t)$ and $C_j(t+1:)$, $C_{j+1}(t+1:)$ be the schedules of $R_j, R_{j+1}$ before and after the traversing of the edge $(v_i,v_{i+1})$ respectively. We replace the schedule $C_j$ for $R_j$ with the schedule $C_j'$ constructed as $C_j(1:t)$ followed by $C_{j+1}(t+1:)$ (see Figure \ref{fig:traverse}. We also replace the schedule $C_{j+1}$ for $R_{j+1}$ with the schedule $C_{j+1}'$ constructed as $C_{j+1}(1:t)$ followed by $C_{j}(t+1:)$. The maximum timespan of schedules $C_j'$ and $C_{j+1}'$ is the same as that of $C_j$ and $C_{j+1}$, and robots $R_j$ and $R_{j+1}$ do not traverse the edge $(v_i,v_{i+1})$ at the timestep $t$. We can proceed the same way with all other pairs of robots traversing the same edge at the same timestep to resolve the collisions caused by traversing until we achieve the collision-free set of schedules $\mathcal{C}'$ that is task-complete and has the same timespan as $\mathcal{C}$.
\end{proof}

    % \begin{figure}
    %     \centering
    %  \includegraphics[width=1\linewidth]{figures/traverse.png}
    %     \caption{The robots' schedules update to avoid traversing the same edge.}
    %     \label{fig:traverse}
    % \end{figure}

In this way, inequalities (\ref{eq:no_crossover}) in IP can be replaced by resolving collisions cased by traversing in a way described in Lemma \ref{lem:traverse} to achieve a collision-free set of schedules of the same timespan.

\section{Greedy and randomized scheduling algorithms}

We start with introducing our greedy scheduling algorithm (GA) that prioritizes assigning the shortest tasks to the nearest robots first. The algorithm works iteratively, starting with the initial set of empty schedules $\mathcal{C}=(C_1,\dots,C_k)$. We denote by $fv_{C_j}$ the final vertex in the schedule $C_j$ if the schedule is not empty, and $sv_j$ otherwise. 
On each iteration of the algorithm, we consider the list $\pairs$ of all possible pairs $(t_i, R_j)$ for all tasks $t_i \in \mathcal{T}$ that are not assigned yet and all robots $R_j$. Then $\pairs$ is sorted in the increasing order of the sum $d_i + \dist(v_i, w_{C_j})$, where $dis$ is the distance between two vertices on the path. We go through the list $\pairs$ and for each pair $(t_i, R_j)$ extend the schedule $C_j$ to the new schedule $C_j'$ by adding the walk from $fv_{C_j}$ to $v_i$ and the task $t_i$ to  $C_j$. We perform the collision-free test on the set of schedules $(C_1,\dots,C_{j-1},C_j',C_{j+1},\dots,C_k)$. If the new set of schedules is collision-free, we set $\mathcal{C}=(C_1,\dots, C_{j-1}, C_j', C_{j+1},\dots, C_k)$, remove from the list $\pairs$ all pairs containing the task $t_i$ and go to the next iteration. If the new set of schedules is not collision-free, we consider the next pair in $\pairs$. The algorithm terminates when the list $\pairs$ is empty, meaning that all tasks are assigned to robots in the set of schedules $\mathcal{C}$. The pseudocode of the algorithm is presented in Appendix \ref{app:implementation}, Algorithm \ref{ref:alg_greedy}.

\begin{figure*}[t]
\scriptsize
\begin{subfigure}{0.32\linewidth}
  \centering
\includegraphics[width=1\linewidth]{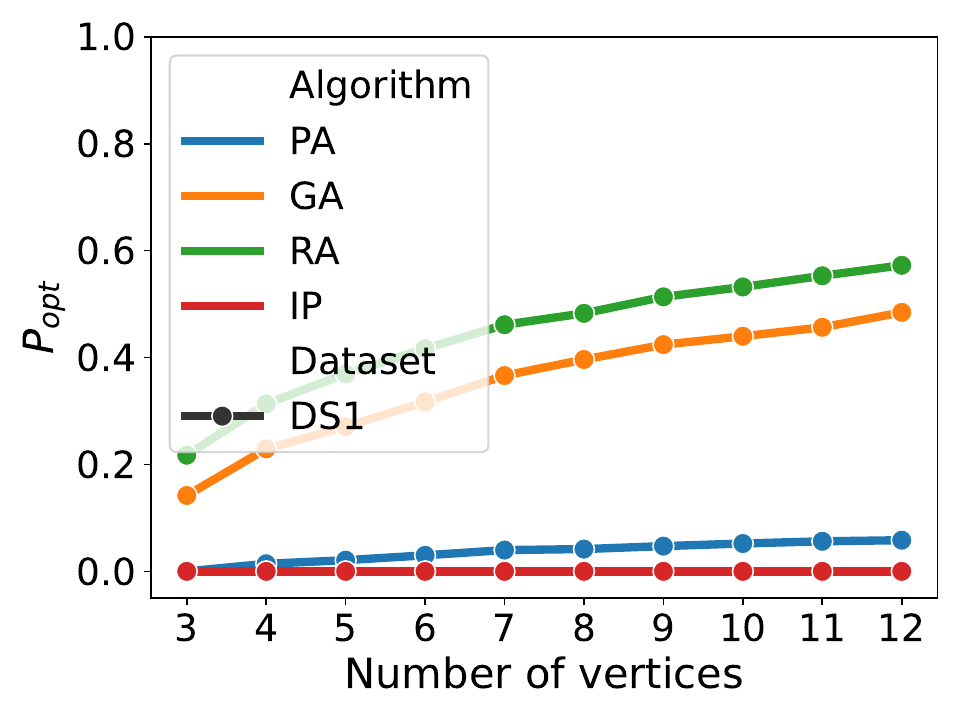}
  % \caption{The evolution of $P_{opt}$ for different algorithms by path length.}
  % \label{fig:non_opt_prop} 
\end{subfigure}
\begin{subfigure}{0.32\linewidth}
\includegraphics[width=1\linewidth]{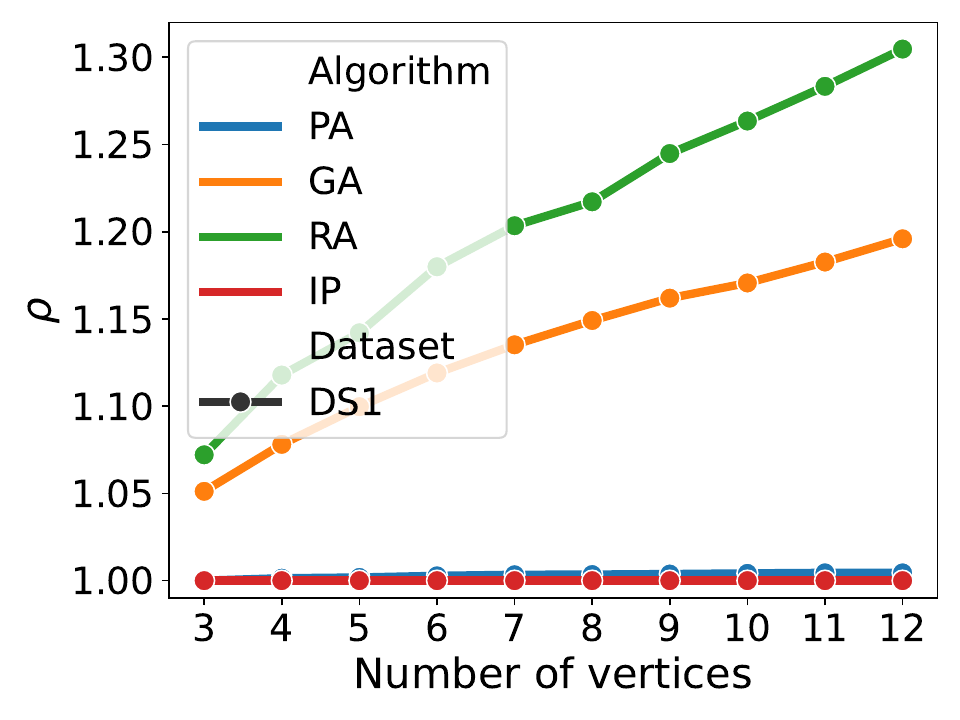}
\end{subfigure}
  \begin{subfigure}{0.32\linewidth}
\includegraphics[width=1\linewidth]{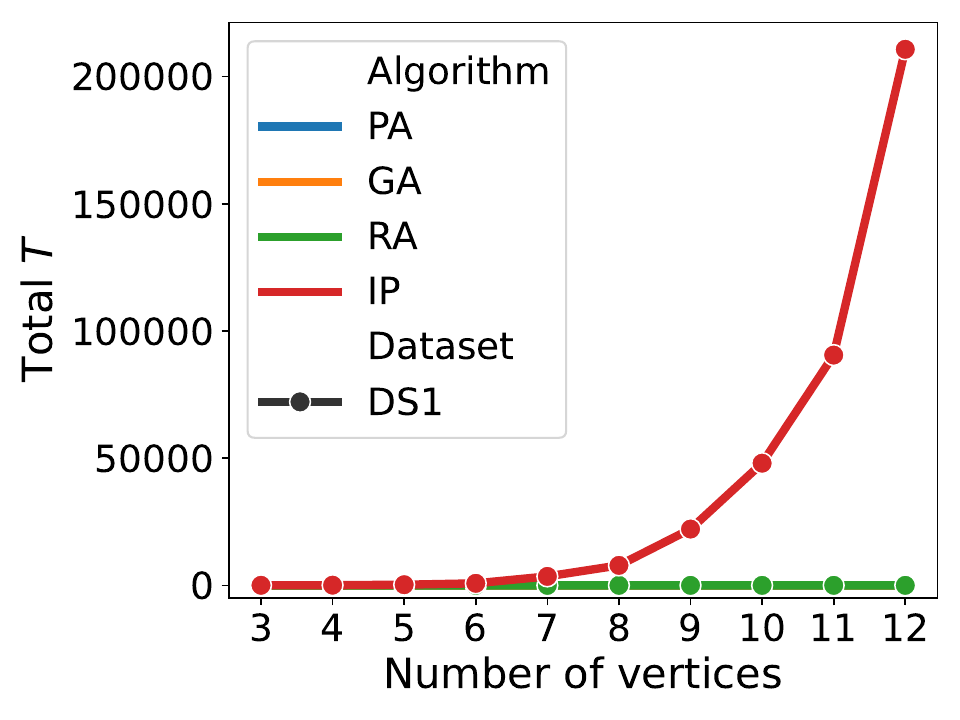}
  \end{subfigure}
  \caption{The evolution of $P_{opt}$ (left), $\rho$ (middle), and total $T$ (right) by the path graph order.}
  \label{fig:non_opt_prop} 
\end{figure*}

\begin{lemma}\label{lem:GA-col-free}
The set of schedules provided by GA is task-completing and collision-free.
\end{lemma}
\begin{proof}
First, we note that GA terminates only when all tasks have been assigned; hence, the set of schedules returned by GA is task-completing. Second, we observe that the algorithm starts with a set of collision-free schedules containing no tasks, and at each step accepts an updated schedule only if the entire set remains collision-free. Therefore, the set of schedules returned by the algorithm is guaranteed to be collision-free.
\end{proof}

\begin{lemma}\label{lem:GA-col-free-test}
The collision-free test in GA can be done in $O(n \cdot m \cdot k \cdot d_{max})$ time, where $n$ is the order of the path graph, $m$ is the number of tasks, $k$ is the number of robots, and $d_{max}$ the upper limit on the duration of a task.
\end{lemma}
\begin{proof}
Assume that we have a collision-free set of schedules $\mathcal{C}=(C_1$, $\dots$, $C_k)$, the task $t_i$ that is not covered by $C$ yet, and the schedule $C_j'$ which is formed by extension of $C_j$ on task $t_i$. We want to check whether the set of schedules $\mathcal{C}'=(C_1$, $\dots$, $C_{j-1}$, $C_j'$, $C_{j+1}$, $\dots$, $C_k)$ is collision-free. Note that the set $\mathcal{C}'$ is collision-free up to the timestep $|C_j|$, hence we need to check the schedules for collision only from the timestep $|C_{j}|+1$. We also know that if $\mathcal{C}'$ is not collision-free, then the collision is between $R_j$ and some other robot. We first check whether $R_j$ is in a collision with any other robot from the timestep $|C_j|+1$ to the timestep $|C_j'|$. We can estimate $|C_j'| - (|C_j|+1)= O(n+d_{max})$ as the maximum possible length of the walk to the task $t_i$ is $n-1$ and $d_i \leq d_{max}$. For each timestep from $|C_j|$ to $|C_j'|$, we check whether $R_j$ inhabits the same vertex or traverses the same edge as any other robot between the current and next timestep. This gives us $O(k(n +d_{max}))$ comparisons, each of which takes a constant time. Then, we want to check whether any other robot inhabits the vertex of $t_i$ after the timestep $|C_j'|$. The maximum possible timespan of the schedule is at most $O(n \cdot m \cdot d_{max})$, thus, as we have $k$ robots, the collision check requires $O(n \cdot m \cdot k \cdot d_{max})$ time.\looseness=-1
\end{proof}

\begin{theorem}\label{th:ga}
Given a path graph of with $n$ vertices, a set of tasks $\mathcal{T}=\{t_1$, $\dots$, $t_m\}$, and the set of robots $R_1,\dots,R_k$, GA returns a collision-free task-completing set of schedules for these robots in $O(m \cdot m^3 \cdot k^2 \cdot d_{max})$ time.
\end{theorem}
\begin{proof}
Following Lemma \ref{lem:GA-col-free}, if GA returns a set of schedules then this set is collision-free and task-completing.
Now, we show that on each iteration, the size of the list $\pairs$ decreases, and the algorithm returns a collision-free schedule after $m$ iterations. To prove this, it suffices to show that on each iteration all pairs containing one of the tasks are removed from $\pairs$. To this aim we show that for a collision-free set of schedules $\mathcal{C}=(C_1,\dots,C_k)$ and any task $t_i$ that is not included in $\mathcal{C}$ yet there exists $R_j$, such that extending the schedule $C_j$ to $C_j'$ by adding $t_i$ will lead to the set of collision free schedules $(C_1,\dots,C_{j-1},C_j',C_{j+1},\dots,C_k)$. Indeed, let $\mathcal{C}=(C_1,\dots,C_k)$ be the set of collision-free schedules and let $t_i$ that is not included in $\mathcal{C}$. We consider the robot $R_j$ closest to $t_i$ at the after completing the set of schedules $\mathcal{C}$. Without loss of generality, we may assume that the last position $fv_{C_j}$ of the robot $R_j$ is at the right to the task $t_i$. We consider the robot $R_{j-1}$ with its schedule $C_{j-1}$ and notice that its last position cannot be between $t_i$ and $fv_{C_j}$, otherwise the distance from $t_i$ to $R_{j-1}$ would be minimum; hence, the position of $t_i$ lies between the last positions of $R_{j-1}$ and $R_j$. Depending on the timespans of the schedules $C_j$ and $C_{j-1}$ we consider two cases.
First, we assume $|C_j| \geq |C_{j-1}|$. In this case, the robots do not change their positions from the timestep $|C_{j}|$ onwards, hence, in the set of schedules $\mathcal{C}$ the part of the part between $R_{j-1}$ and $R_{j}$ is free of robots after the timestep $|C_{j}|$. Adding $t_i$, which lies in this part of the path, to the schedule of $R_j$ leads to the extended schedule $C_j'$ and the set of schedules $(C_1,\dots,C_{j-1},C_j',C_{j+1},\dots,C_k)$ is collision-free.
Now, let $|C_{j}| < |C_{j-1}|$. Following the same arguments as before, we can show that the robots do not change their positions from the timestep $|C_{j-1}|$ onwards and that adding $t_i$ to the schedule of $R_{j-1}$ allows to get the extended schedule $C_{j-1}'$ and the collision-free set of schedules $(C_1,\dots,C_{j-2},C_{j-1}',C_{j},\dots,C_k)$. We conclude that on each iteration where exists a pair $(t_i,R_j)\in pairs$ such that including $t_i$ into the schedule of $R_j$ leads to the updated set of schedules that will pass the collision-free test.

Finally, we turn to the algorithm's work time. On each of the $m$ iterations, we first sort the $\pairs$ list of $O(m\cdot k)$ elements for $O(m\cdot k \log(m\cdot k))$ time and then for at most $mk$ times perform the collision-free test which, by Lemma \ref{lem:GA-col-free-test}, can be done in $O(n \cdot m \cdot k \cdot d_{max})$ time. Hence, the algorithm will return the set of schedules in $O(n \cdot m^3 \cdot k^2 \cdot d_{max})$ time.
\end{proof}

We complete this section by introducing the randomized scheduling algorithm (RA), which repeats the logic of GA except that the order of the list $\pairs$ in RA is randomised rather than ordered as in GA. As Lemmas \ref{lem:GA-col-free} and \ref{lem:GA-col-free-test}, as well as Theorem \ref{th:ga}, do not use the order of elements in $\pairs$ in their proofs, they are applicable to RA as well as GA. For conciseness, we provide the pseudocode of RA as Algorithm \ref{ref:alg_greedy} in Appendix \ref{app:implementation}.

\begin{figure*}[t]
  \centering
  \scriptsize
  \begin{subfigure}{0.26\linewidth}
\includegraphics[width=1\linewidth]{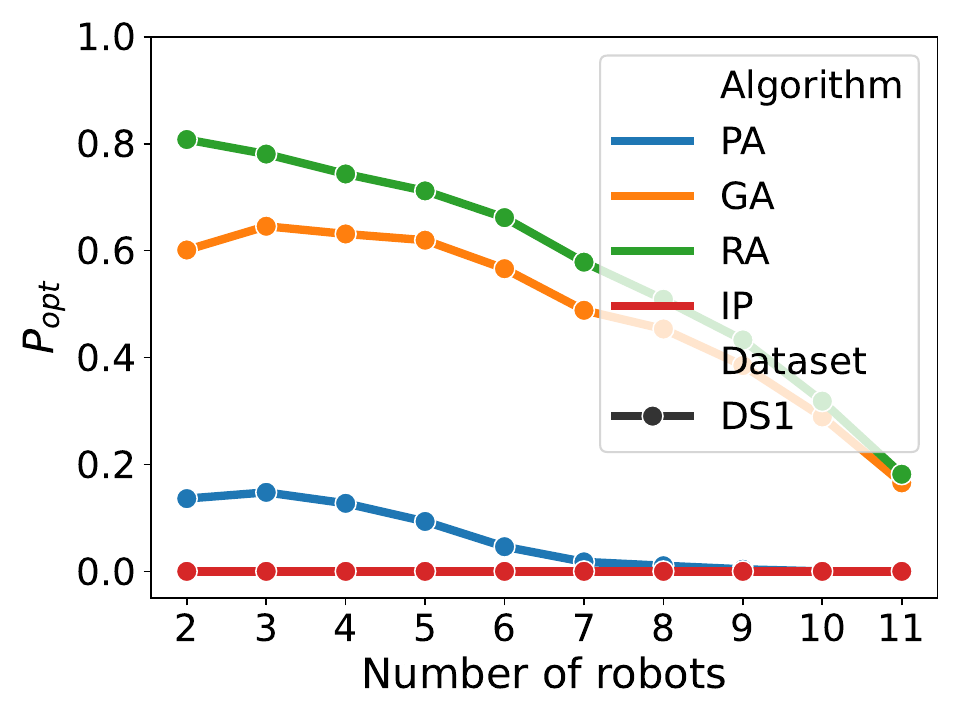}
  \end{subfigure}
  \begin{subfigure}{0.26\linewidth}
\includegraphics[width=1\linewidth]{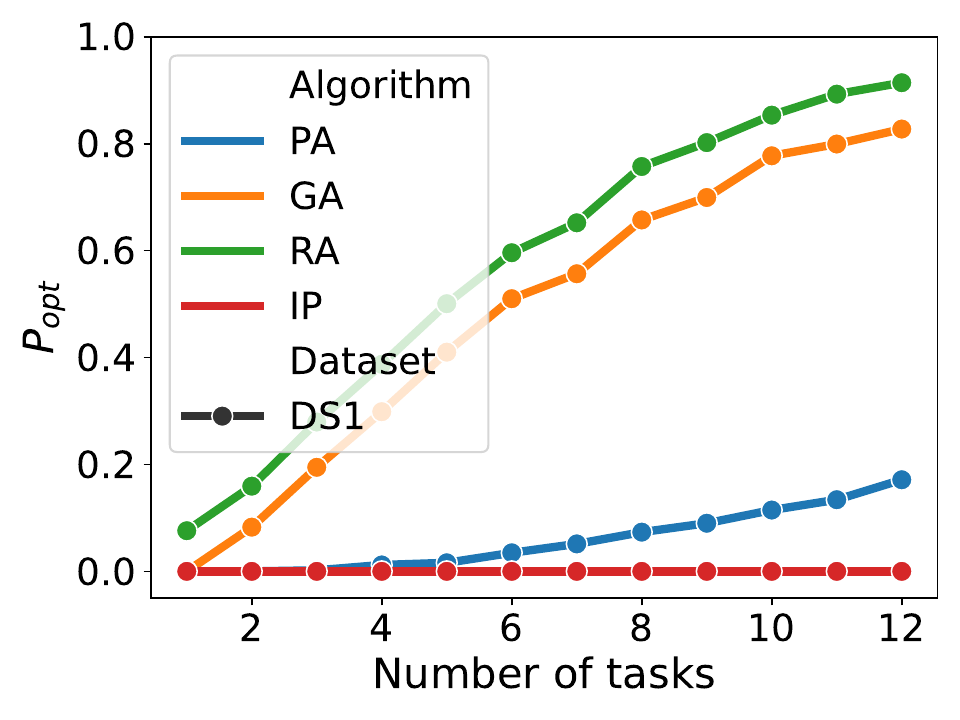}
  \end{subfigure}
  \begin{subfigure}{0.26\linewidth}
\includegraphics[width=1\linewidth]{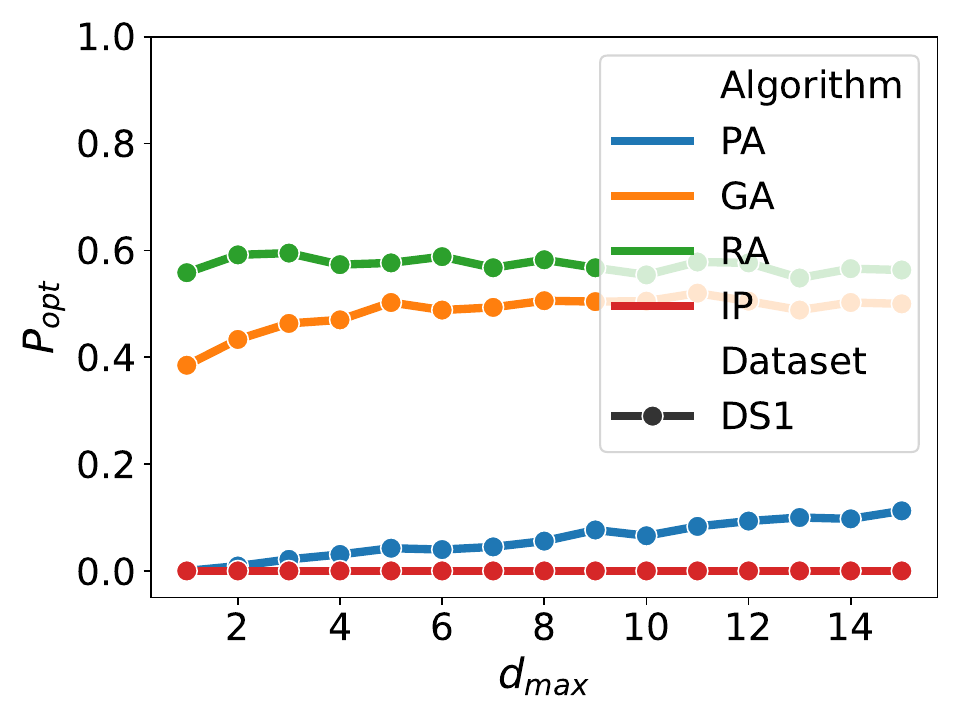}
  \end{subfigure}
  % \begin{subfigure}{0.49\linewidth}
  % \includegraphics[width=1\linewidth]{figures/prop_not_equal_['pa', 'ga', 'ip']_by_mean_task_length.pdf}
  % \end{subfigure}
\caption{The average $P_{opt}$ for $n=12$, for different numbers of robots (left), numbers of tasks (middle), and the maximum task duration $d_{max}$ (right).}
  \label{fig:non_opt_prop_max_l} 
\end{figure*}

\section{Experiments}
%TOTAL:12767374
%OPTIMAL: 11061661
%Percentage=86.64006396303579%

In this section, we perform an experimental analysis of $PA$ comparing it against various scheduling algorithms by their efficiency, scalability, and weaknesses.
The efficiency of a scheduling algorithm can be measured using various metrics.
The first metric we use is the proportion of the fastest schedules ($P_{opt}$) provided by an algorithm $A$.
We compute $P_{opt}$ for $A$ by comparing the timespans of the schedules produced by $A$ with those provided by IP, since IP is the only scheduling algorithm in our comparison that guarantees the optimality of the solution.
Next, we introduce the relative performance $\rho$ of $A$, defined as the timespan of the schedule obtained from $A$ divided by the optimal timespan and averaged.
Finally, we compare the algorithms based on the average (per instance) and total time $T$ required by the algorithms.\looseness=-1

\paragraph*{Scheduling on the path graphs of the order $\leq$~12}

We start with the experiment for small numbers: we consider the path graph of order $n$ from 3 to 12, number of tasks $m$ from 1 to $n$, and task durations drawn from uniform distributions ranging from 1 to $d_{max}$, where $d_{max}$ varies from 1 (all tasks have duration 1) to 15.\looseness=-1
% We notice that here the definition of $d_{max}$ slightly differs from that suggested above, as the maximum number for the uniform distribution on the durations of the tasks does not guarantee that there will be a task with this exact duration, but for our purposes the new definition fits well.

For each tuple $(n, m, d_{max})$, we create 10 path graphs with tasks on them, where all tasks allocated following the uniform distribution. For each graph and set of tasks, we consider $k$ robots allocated uniformly and vary $k$ from 2 to $n-1$. For the graph with 3 vertices, for some combinations of $m$ and $d_{max}$, there are no 10 different sets of tasks, and in this case, we consider the maximum number of different sets of tasks. This gives us the dataset $DS1$ that contains $74,251$ elements in total.
We run PA, IP, GA, and RA on $DS1$ and compare the results by $P_{opt}$, $\rho$, and average and total $T$.
Table \ref{tab:metrics_20} shows the average values for all metrics achieved in this dataset. In the following subsections, we analyse the results from Table \ref{tab:metrics_20} in detail.\looseness=-1

\begin{table}[h]
    \centering
    \scriptsize
    \caption{Performance comparison of scheduling algorithms in the experiment with $DS1$.}
    \label{tab:metrics_20}
    \begin{tabular}{c|c c c c}
    \toprule
    Metric& IP&PA& GA & RA\\
    \midrule
    $P_{opt}$ & 0 & 0.049 & 0.426 & 0.518 \\
    $\rho$ & 1 & 1.004 & 1.167 & 1.255 \\
    Average $T$ & 5.17 & 0.00043 & 0.00017 & 0.00017 \\
    Total $T$ & 383,912.89 & 31.58 & 12.36 & 12.75 \\
    \bottomrule
    \end{tabular}
\end{table}

\noindent
\textbf{Proportion of non-optimal solutions.}

\noindent
We start our analysis with the $P_{opt}$ metric and notice that the proportion of non-optimal solutions for PA is very low (0.049) and approximately 10 times less than other non-optimal algorithms, namely, GA (0.426) and RA (0.518). We then compare evolutions of $P_{opt}$ with the growth of the path length (Figure \ref{fig:non_opt_prop}).
For all methods, $P_{opt}$ increases as the path graph order grows; however, $P_{opt}$ for PA increases slower than the other methods and remains relatively low even at the maximum number of vertices.\looseness=-1

We proceed with the subset of $DS1$ with instances of the maximum graph order $n=12$ and explore how other parameters affect $P_{opt}$ for different methods (see Figure \ref{fig:non_opt_prop_max_l}).
We can see that increasing the number of tasks leads to an increase in $P_{opt}$ for all algorithms, with PA showing consistently better results.
For the number of robots, the correlation is reversed: the highest 
$P_{opt}$ occurs in instances with the smallest number of robots (2 or 3) for all non-optimal algorithms.
Regarding the various tasks' durations, both GA and RA show a high average $P_{opt}$ for all values, while PA shows significantly lower $P_{opt}$ that slowly increases with the growth of $d_{max}$.
These findings show that PA consistently outperforms GA and RA in the $P_{opt}$ metric and is closest to IP, which always provides a fastest schedule.

\noindent
\textbf{Relative performance.}
% 
% \noindent
Observe that the relative performance $\rho$ of PA is 1.004, meaning that on average the schedule suggested by PA is just 0.4\% longer than a fastest schedule, in contrast to the GA and RA that give 16.7\% and 25.5\% longer schedules on average. The evolution of $\rho$ with the growth in the number of vertices in Figure \ref{fig:non_opt_prop} confirms that, on average, the timespans of the schedules provided by PA stay close to the optimal timespans provided by IP, while the lengths of schedules provided by GA and RA increase faster.
As with the $P_{opt}$ metric, we also investigate the maximum $n=12$ and show how the number of robots, the number of tasks, and the durations of tasks affect the $\rho$ metric (see Figure \ref{fig:frac_max_l}). Here, we have a similar to $P_{opt}$ trend: $\rho$ of PA is significantly less than those of GA and RA, slowly grows with the tasks number and maximum task duration growth, and reduces with the robots' number growth.
We conclude that while the proportion of non-optimal solutions for PA grows when we scale the problem, the timespans of these non-optimal solutions remain very close to the optimal timespan.\looseness=-1

\begin{figure*}[t]
  \centering
  \scriptsize
  \begin{subfigure}{0.26\linewidth}
\includegraphics[width=1\linewidth]{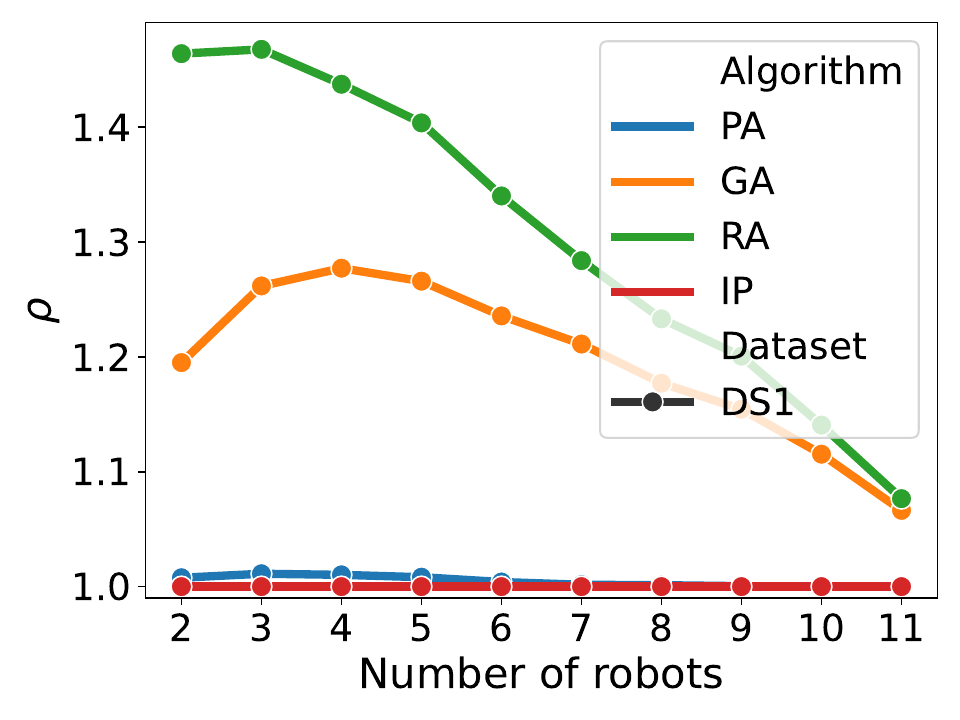}
  \end{subfigure}
  \begin{subfigure}{0.26\linewidth}
\includegraphics[width=1\linewidth]{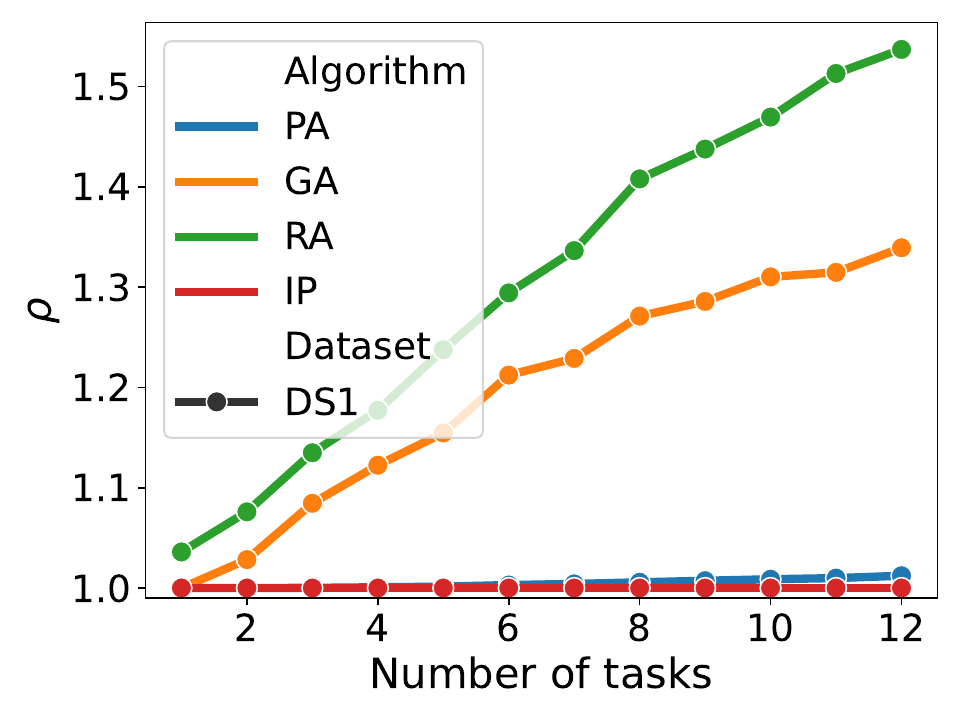}
  \end{subfigure}
  \begin{subfigure}{0.26\linewidth}
\includegraphics[width=1\linewidth]{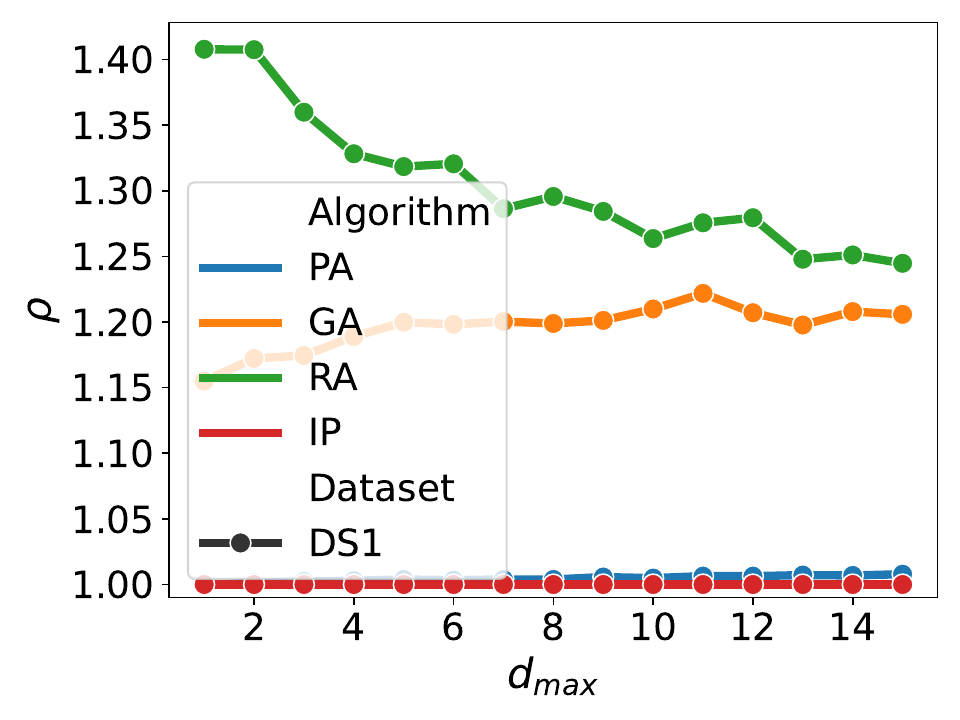}
  \end{subfigure}
  \caption{The relative performance $\rho$ for $DS1$ and $n=12$, for different numbers of robots (left), numbers of tasks (middle), and the maximum task duration $d_{max}$ (right).}
  \label{fig:frac_max_l} 
\end{figure*}
\begin{figure*}
  \scriptsize
  \centering
  \begin{subfigure}{0.24\linewidth}
\includegraphics[width=1\linewidth]{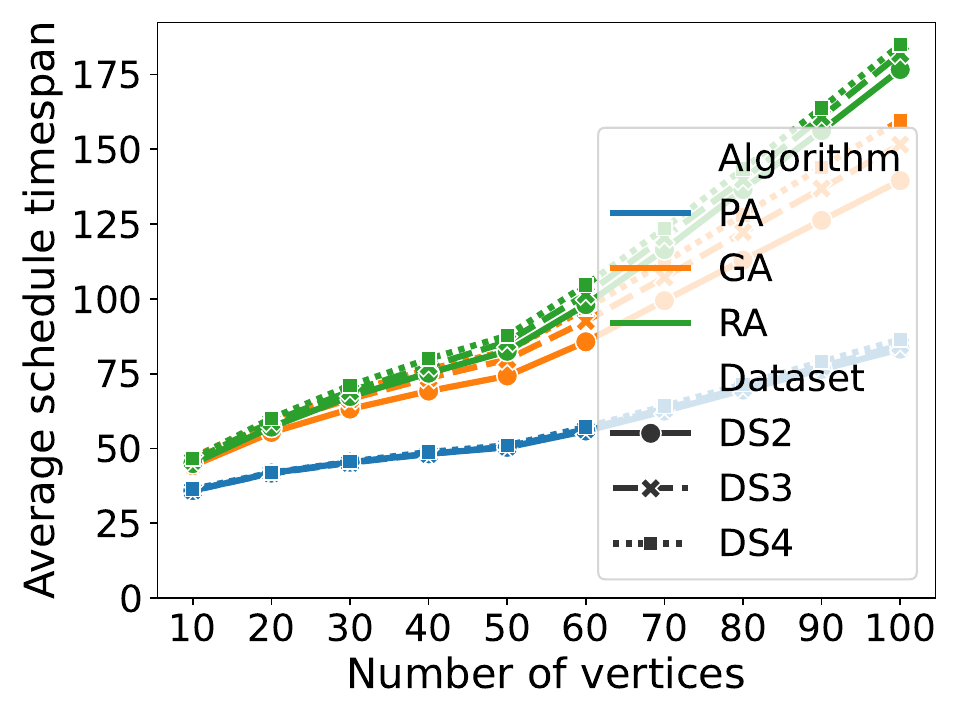}
  \end{subfigure}
  \begin{subfigure}{0.24\linewidth}
\includegraphics[width=1\linewidth]{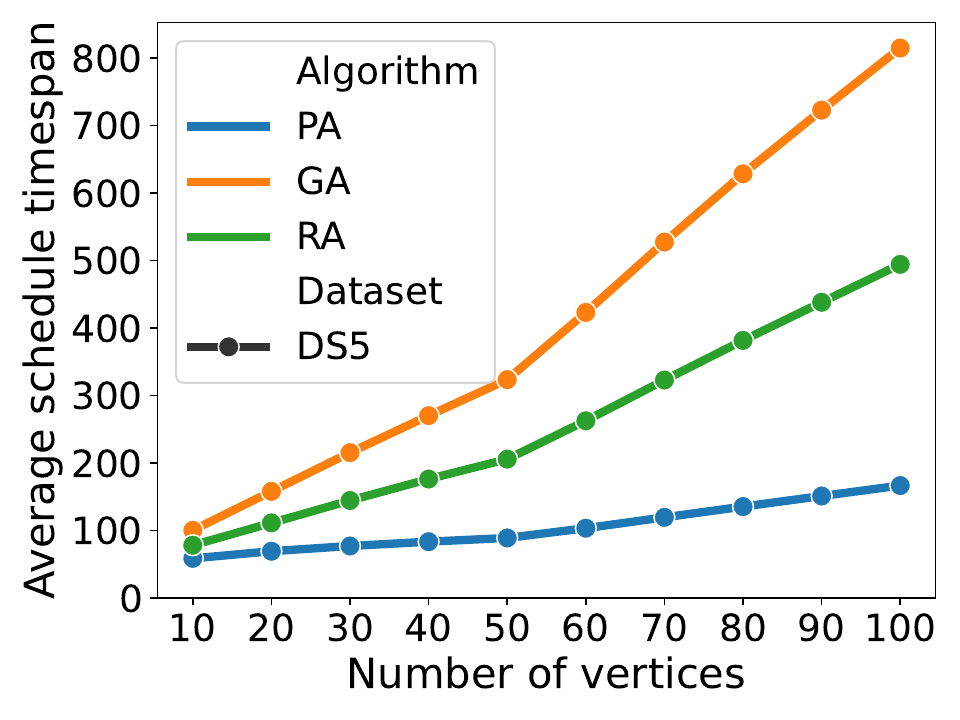}
  \end{subfigure}
  \begin{subfigure}{0.24\linewidth}
\includegraphics[width=1\linewidth]{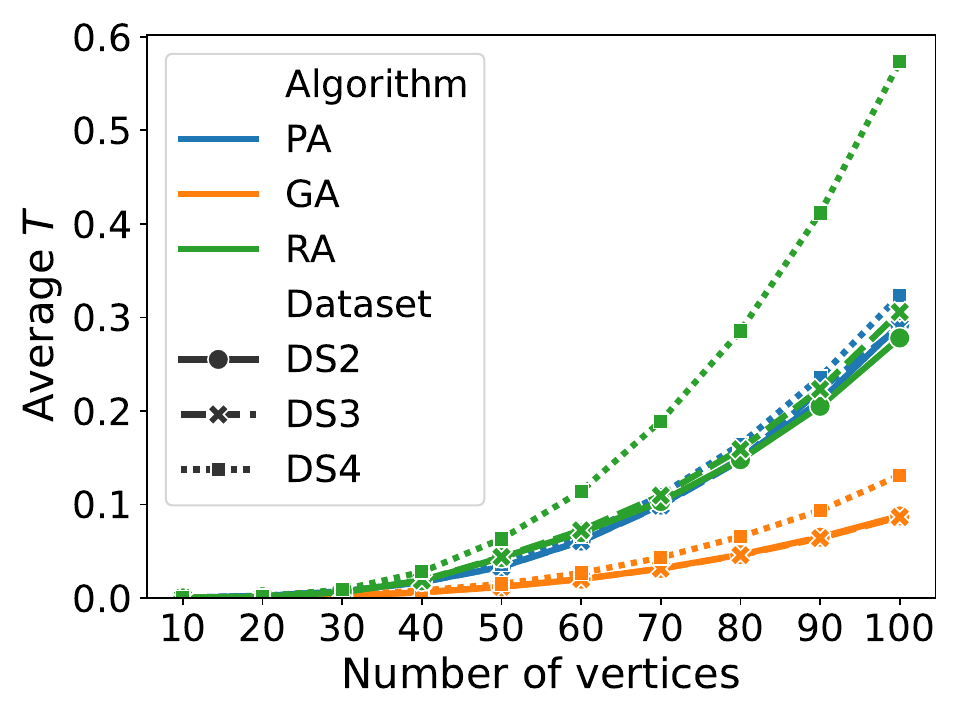}
  \end{subfigure}
  \begin{subfigure}{0.24\linewidth}
\includegraphics[width=1\linewidth]{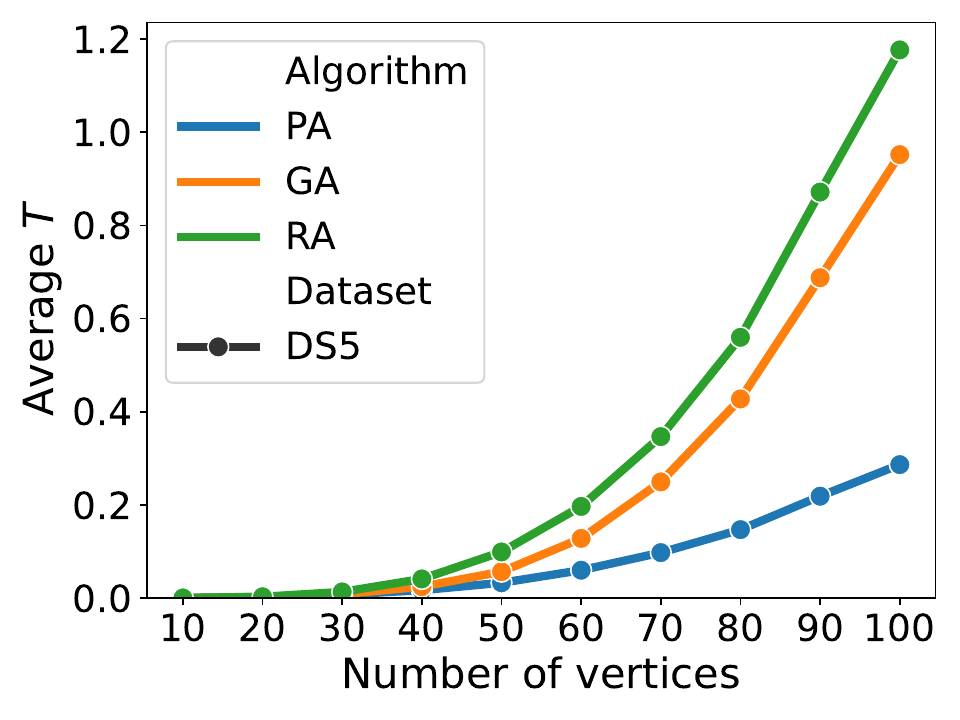}
  \end{subfigure}
  \caption{The average schedule timespan (left) and the average $T$ (right) for different algorithms by the number of vertices for $DS2$-$DS4$ and $DS5$.}
  \label{fig:path_length_1000} 
\end{figure*}

\noindent
\textbf{Computation Time.}
%
% \begin{figure}
%   \centering
% 
%   \begin{subfigure}{0.45\linewidth}
%   \includegraphics[width=1\linewidth]{figures/Average time_['pa', 'ga', 'ra', 'ip']_by_path_length.pdf}
%   \end{subfigure}
%   \caption{The total and average time (in sec.) for the different path lengths.}
%   \label{fig:time} 
% \end{figure}
%
% \noindent
We complete the analysis of the experiment with $DS1$ comparing the methods by the average~$T$ per schedule, and total $T$ spent during the experiment. 
We first notice that IP is at least four orders of magnitude slower than all other methods in scheduling for $DS1$. Moreover, both the total and the average $T$ of IP rapidly grow with the growth of the order of a graph (see Figure \ref{fig:non_opt_prop}). In one particular instance with $n=12$ presented in Appendix~\ref{app:example}, IP required more than 4 hours to find the optimal schedule of timespan 63, whereas PA and GA each produced a schedule of timespan 66 in 0.0006 and 0.017 seconds, respectively. This confirms that the use of IP for task scheduling is inefficient in practice at scale, despite the guaranteed optimality of the resulting schedule.
As such we exclude IP from further large-scale experiments and focus on the performance comparison between non-optimal methods.\looseness=-1

\paragraph*{Scheduling on the longer paths}

The design of the datasets for the scaled experiment is as follows: we consider the path graphs of order $n$ in the range from 10 to 100 with step 10, the number of tasks $m$ from 2 to $n$ with step 2, and the maximum task duration $d_{max}$ from 10 to 50 with step 5. For each combination of graph and tasks, we consider the number of robots $k$ from 2 to 50 in steps of 2. In this experiment, we also consider different variations in the distributions of tasks and robots along the path, as well as variations in task durations, to represent various realistic scenarios. Based on these distributions, we generate four different datasets of the same size ($567,000$ instances).\looseness=-1

The dataset $DS2$ follows the same setup as $DS1$. The dataset $DS3$ represents a case in which tasks of similar duration are located close to each other. Specifically, in $DS3$, the path is divided into two equal parts: one containing all tasks with durations shorter than $\frac{d_{max}}{2}$, and the other containing all longer tasks. The dataset $DS4$ represents a case where certain regions of the path contain a higher density of tasks. Here, the task locations follow a normal distribution with a mean at a randomly selected vertex of the path and a standard deviation of $n/8$. In $DS5$, robots are clustered around a specific point on the path and distributed according to the same normal distribution as the tasks in $DS4$. Table \ref{tab:distr} in Appendix \ref{app:implementation} shows how the new datasets $DS2$-$DS5$ are organized.\looseness=-1

As we excluded IP from this experiment, the minimum possible schedule timespan is unknown, and we cannot compare the algorithms using $P_{opt}$ or $\rho$. Instead, we compare the timespans of the sets of schedules directly -- referred to hereafter simply as the schedule timespan -- and also consider the scheduling time $T$.
Table \ref{tab:sch_length} presents the average schedule timespans achieved by PA, GA, and RA on datasets $DS2$–$DS5$. The results clearly show that PA outperforms both non-optimal baselines, achieving significantly shorter schedules across all datasets. On average, the sets of schedules produced by PA are more than twice as short as those generated by GA and RA, with GA outperforming RA.\looseness=-1

Interestingly, despite its simplicity, RA does not always exhibit the worst performance. Specifically, in the case of $DS5$, RA significantly outperforms GA. In $DS5$, the robots are distributed according to a normal distribution and tend to cluster in a particular region along the path. In this scenario, a robot may effectively divide the path into two sections and claim all tasks from one side, as it is the closest to them, thereby blocking access for robots located on the other side of the path. The greedy strategy, in this case, substantially increases the schedule timespan compared to RA, where tasks are allocated independently of their distance to the robots, as demonstrated in Appendix~\ref{app:example}.\looseness=-1

\begin{table}[h]
    \centering
    \scriptsize
    \caption{The average schedule timespan and the average time per schedule achieved by various algorithms for different datasets. The \textbf{AVERAGE} row presents numbers averaged by all datasets. The best results, in terms of both timespan of the solution and average computation time are bold.}
    \label{tab:sch_length}
    \begin{tabular}{c|c c c|c c c}
    \toprule
    {} & \multicolumn{3}{c|}{Average schedule timespan} & \multicolumn{3}{c|}{Average $T$} \\
    \midrule
    Dataset & PA & GA & RA & PA & GA & RA \\
    \midrule
    $DS2$ & \textbf{66} & 106 & 128 & 0.14 & \textbf{0.04} & 0.14 \\
    $DS3$ & \textbf{67} & 115 & 132 & 0.14 & \textbf{0.04} & 0.16\\
    $DS4$ & \textbf{68} & 121 & 135 & 0.16 & \textbf{0.06} & 0.28\\
    $DS5$ & \textbf{129} & 571 & 350 & \textbf{0.14} & 0.44 & 0.56 \\
    \midrule
    \textbf{AVERAGE} & \textbf{82} & 228 & 186 & \textbf{0.14} & \textbf{0.14} & 0.28\\
    \bottomrule
    \end{tabular}
\end{table}

The evolution of the average schedule timespan with increasing numbers of vertices, shown in Fig.~\ref{fig:path_length_1000}, demonstrates that the average schedule timespan increases with $n$ for all algorithms and datasets. However, for PA, the average schedule timespan remains consistently lower than that of GA and RA across all datasets and values of $n$.

In Appendix \ref{sec:stick_100}, we analyze how the number of tasks, the number of robots, and the maximum task duration affect the performance of the algorithms. The results confirm that PA scales better than GA and RA with respect to the number of tasks and the growth of task durations, and it outperforms both baselines regardless of the number of robots.

\noindent
\textbf{Performance at fixed task-to-robot ratios.}
We investigate the case where the numbers of robots and tasks increase simultaneously while maintaining a fixed ratio between them. From datasets $DS2$–$DS5$, we extract the instances with $n=100$ where this ratio equals 1, 2, 5, or 10. We denote by $DSi_j$ the subset of dataset $DSi$ with the ratio equal to $j$. Figure \ref{fig:fixed_ratio_DS5} shows how the average schedule timespan produced by the different algorithms for the subsets of $DS5$ with fixed ratios evolves as the number of robots (and tasks) increases. We observe that, in all cases, PA consistently outperforms the baselines. Moreover, while the average schedule timespan of PA remains the same or decreases as the number of robots (and tasks) grows, for GA and RA, the average schedule timespan increases linearly and can become up to eight times longer than that of PA for the maximum number of robots and tasks. This indicates that PA effectively utilizes the increasing number of robots and scales better than GA and RA. A significant outperformance of PA over GA and RA on the subsets of $DS2$–$DS4$ is discussed in detail in Appendix \ref{app:fixed_ratio}.\looseness=-1

\begin{figure}
  \centering
  \begin{subfigure}{0.24\linewidth}
\includegraphics[width=1\linewidth]{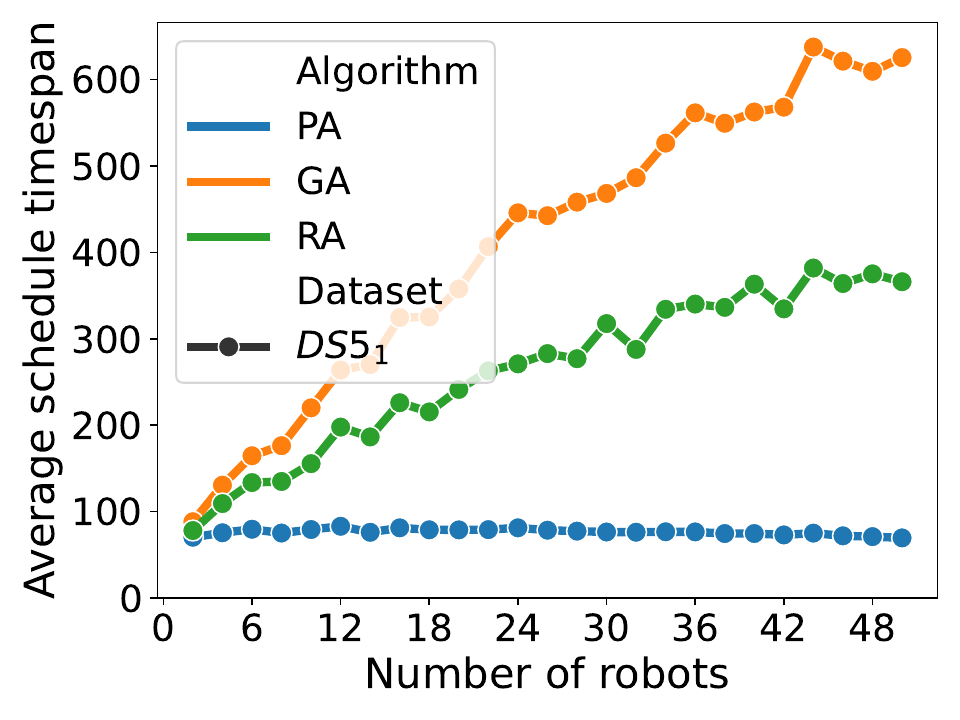}
  \end{subfigure}
  \begin{subfigure}{0.24\linewidth}
\includegraphics[width=1\linewidth]{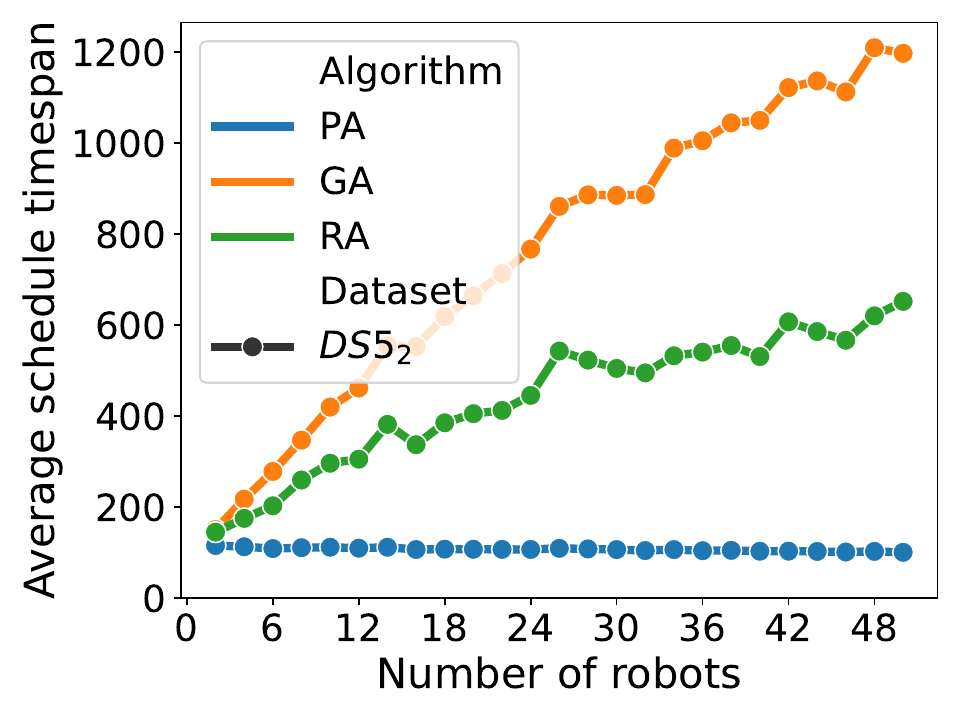}
  \end{subfigure}
  \begin{subfigure}{0.24\linewidth}
\includegraphics[width=1\linewidth]{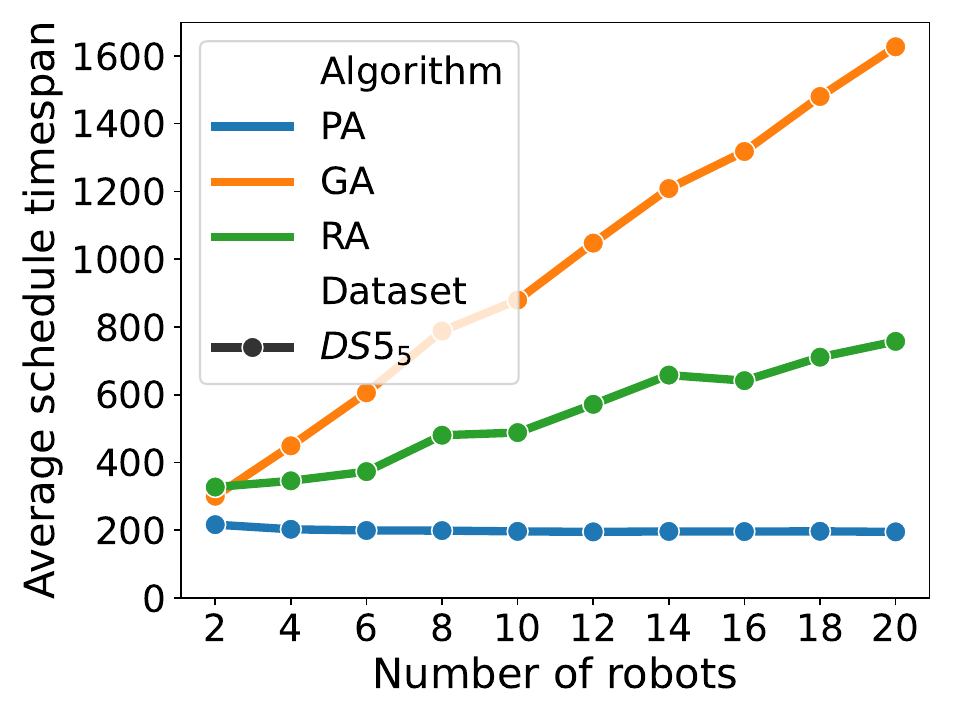}
  \end{subfigure}
  \begin{subfigure}{0.24\linewidth}
\includegraphics[width=1\linewidth]{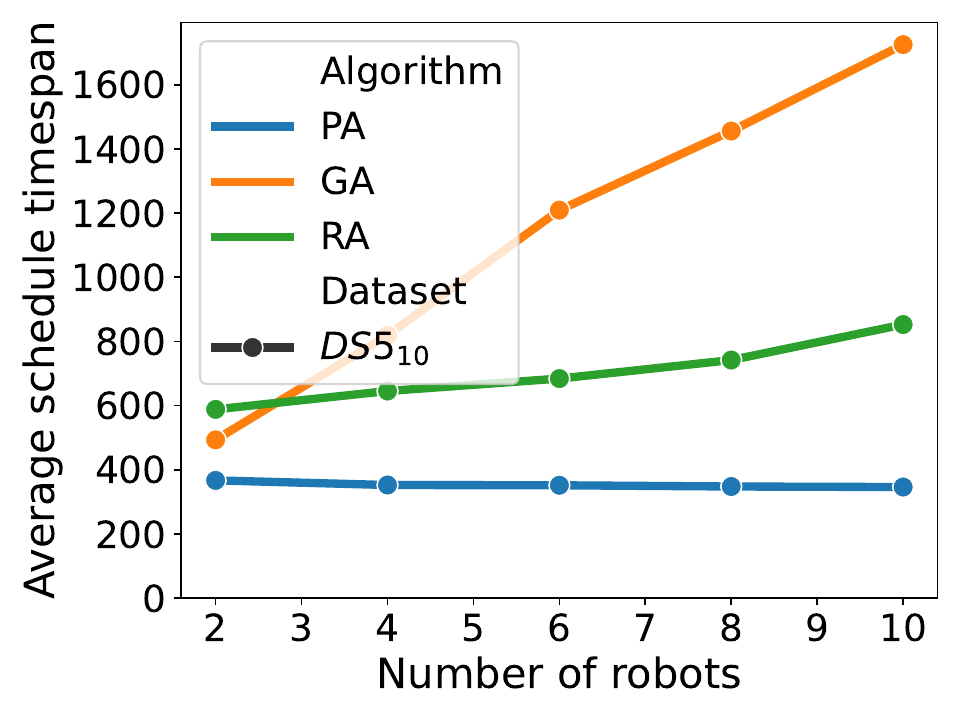}
  \end{subfigure}
  \caption{Experiments $DS5$ subsets with task-to-robot ratios (from left to right): 1, 2, 5, 10.}
  \label{fig:fixed_ratio_DS5} 
\end{figure}

\noindent
\textbf{Scheduling time.}
We compare the average $T$ for scheduling performed by PA, GA, and RA on the datasets $DS2$-$DS5$ in Table \ref{tab:sch_length}. All three algorithms exhibit a good speed of work, and there is no one particular winner for all datasets; however, PA and GA, on average, are twice as fast as RA. We also notice that none of the three algorithms have required more than 12 seconds to find a collision-free schedule for any instance in the $DS2$-$DS5$ datasets, in contrast to IP, which required up to 4 hours for the much shorter path. This confirms that all three non-optimal scheduling algorithms are efficient from the time of work perspective.\looseness=-1

\section{New Hardness Results Beyond Path Graphs} 

% TODO: EXPAND HARDNESS RESULTS, MOVE TO END OF THE PAPER

\label{sec:path_connecting_tasks_hardness}

In this section, we provide a new hardness result for $k$-\rsp on grid graphs, showing that even for very simple graphs, $k$-\rsp remains a challenging problem, and further explaining our focus on path graphs.

\begin{theorem}
    \label{thm:k-rsp_hardness_grids}
    $k$-\rsp is NP-hard even when the input graph is a grid, all tasks have equal duration, and $k = 1$.
\end{theorem}

\begin{proof}
    Note membership in NP is given by \cite{ourPreviousPaper}. To show hardness for grids, we first note the hardness of the Hamiltonian Path problem on subgrids, graphs that can be \emph{embedded} on grids, shown by Itai et al. \cite{Itai1982}. Formally, the graph $G = (V, E)$ can be embedded in the graph $G' = (V', E')$ if there exists some mapping $f: V \mapsto V'$ such that if $(v, u) \in E$, then $(f(v), f(u)) \in E'$.
    Given an instance of \textsc{Hamiltonian Path} on the subgrid $G = (V, E)$, where $G$ can be embedded on the $n \times m$ grid $G' = (V', E')$, we construct and instance of $1$-\rsp on $G'$. Let $S \subseteq V'$ be the set of vertices corresponding to the embedding of $G$ into $G'$. We construct our new instance of $1$-\rsp by constructing a set of $\vert V \vert$ tasks of duration $1$, with each vertex in $S$ containing a task. We choose, as our starting vertex, some arbitrary vertex in $S$. Observe that there is a solution to the $1$-\rsp instance with a task-completing schedule of length $2\vert V \vert - 1$, if and only if there is a Hamiltonian path in $G'$. In the other direction, any Hamiltonian path on $G'$ can be converted into a task-completing schedule in the $1$-\rsp instance by having the robot complete the task on its starting vertex, then moving along the Hamiltonian path, completing the task on each vertex as it visits it. Thus we have the statement.\looseness=-1
\end{proof}

We strengthen this result by looking at the problem of finding a path connecting all tasks in a graph. We note that, if such a path can be found easily, then we may use these paths along with PA as a heuristic for robot scheduling on a general graph.\looseness=-1

% Since an algorithm for approximating a solution to the \textsc{$k$-Robot Scheduling} problem on a path was given in the past , a natural sub-problem is in general graphs - does there exist a path connecting all tasks, and if so, can we find one?\looseness=-1
% We use this concept in subsection \ref{sec:grids} when using the partition algorithm on a lattice graph. 
\begin{problem}[\textsc{Path}]
    Given a graph $G=(V,E)$ and a set $S \subseteq V$, does there exist a simple path connecting all vertices in $S$?
\end{problem}
%The problem {\sc path} is : given a graph $G=(V,E)$ and a set $S \subseteq V$, is there a path connecting all vertices in $S$ and if yes , can we find one such path ?

Note that this problem is NP-Complete by reduction from the Hamiltonian Path problem.

\begin{theorem}\label{thm:path_NPC}
The problem {\sc path} is NP-Complete, even when the input graph is a subgrid.
\end{theorem}

% \begin{proof}
% It is clear that \textsc{path} can be verified in polynomial time by traversing the certificate path to ensure $S$ is covered. The hardness follows directly from the NP-hardness of the Hamiltonian path problem on subgrids, by setting $S = V$. Note that we can make $S$ an arbitrarily small fraction of $V$ by simply adding a path of appropriate length to the graph $G$ connected to some vertex.
% \end{proof}

Despite the hardness result, we note that there exists a fixed-parameter tractable algorithm, parameterised by the number of vertices in the set $S$.

\begin{theorem}\label{thm:path_polytime}
    The problem {\sc path} can be solved in polynomial time for fixed $m = \vert S \vert$. %fixed parameter tractable with respect to the parameter $k = \vert S \vert$.
\end{theorem}

\begin{proof}
Any path (if it exists) connecting all vertices in $S$ will visit them in some particular order. We shall examine all possible orderings of the $m$ vertices. Let such an ordering be $v_1, v_2,\dots,v_k$. 
We construct a graph $G'$ such that if there exists $m$ disjoint paths in $G'$ then we have a path connecting all vertices in $S$ in $G$.
 For each vertex $v_i \in S$ we add the pair of vertices $(v_i, v'_i)$ to $G'$. Where $v'_i$ is adjacent to both $v_i$ and the neighbourhood of $v_i$ in $G$. 
We define two vertex sets, $A$ and $B$, with $A = \{v_1,v'_1,v_3,v'_3,\dots\}$ being all odd-numbered vertices and $B = \{v_2,v'_2,v_4,v'_4,\dots\}$ being all the even-numbered vertices.
We also add two vertices $s$ and $t$ to $G'$, with $s$ being adjacent to all vertices in $A$ (other than $v_1$) and $t$ being adjacent to all vertices in $B$.
The problem is now whether or not there exist $m$ internally disjoint paths between $s$ and $t$. This can be done in polynomial time for fixed $m$ by using the $O(n^{m^{5^m}})$ algorithm in \cite{lochet2021} with the graph $G'$ and $m$ copies of the pair $(s, t)$.
In the yes instance, the $m$ disjoint paths will be of the form $(s,v'_1,\dots,v_2,t)$, $s(,v_3,\dots,v'_2,t)$, $(s,v'_3,\dots,v_4,t)$, etc. Hence, in the original graph $G$ we have the path between all vertices in $S$ in the order  $v_1,\dots,v_m$. 
Since we need to do this for all orderings of the $m$ vertices in $S$ this can be done in $O(m!\cdot n^{O(m^{5^m})})$ time.\looseness=-1
\end{proof}

% \begin{figure}
%     \centering
%  \includegraphics[width=0.5\linewidth]{figures/traverse.pdf}
%     \caption{An example of the robots' schedules update to avoid traversing the same edge, the robots are coloured green and blue.}
%     \label{fig:traverse}
% \end{figure}

%\begin{theorem}
    %The problem {\sc path} is W[1]-hard, parametrised by $k = \vert S \vert$
%\end{theorem}

%\section{Shortest Paths between Tasks}
%If the answer to problem {\sc path} is yes then a subsequent challenge is to efficiently find a shortest path connecting all tasks in $S$.
%

\section{Conclusion}

In this paper, we studied the $k$-robot scheduling problem on the path graphs. %We first provided some theoretical results about a motivating sub-problem (to test if a path exists connecting all tasks in a general graph).
We formulated and proved the correctness of an integer programming model for collision-free robots scheduling on paths.
We presented three scheduling algorithms and performed a detailed experimental analysis and comparison with PA, the method proposed in \cite{ourPreviousPaper}. We showed that for the paths of length up to 12 PA matches an optimal solution for Problem \ref{prob:robot-scheduling} on path graphs in 95.1\% cases and provides collision-free sets of schedules only 0.04\% longer than the optimal timespan, significantly outperforming other non-optimal methods in both metrics. Moreover, PA is at most four orders of magnitude faster than the only method that provides an optimal solution, meaning that PA is more applicable in practice. For the experiments at the scale where no optimal solution was practical to look for, we showed that PA, on average, provides a set of schedules that has half the timespan of those returned by GA and RA, and demonstrates an excellent scalability with respect to the key parameters such as the number of vertices, the number of tasks, and the maximum task duration, as well as utilizes the increasing number of robots more wisely. 
We also tested PA on the datasets with variable tasks positions distributions, robots positions distributions, and tasks durations' distributions as well as different fixed task-to-robot ratios. In all cases, PA consistently outperforms other non-optimal methods in the average schedule timespan, while all non-optimal methods can find a solution in the matter of seconds.
We also provided some new hardness results for the problem of scheduling on grids.
We conclude that PA is a powerful scheduling method for path graphs that is capable to provide a fast collision-free set of schedules for a fraction of time.\looseness=-1

% This paper has shown that the Partition Algorithm given in \cite{ourPreviousPaper} matches the optimal solution for Problem \ref{prob:robot-scheduling} in many cases, particularly for small numbers of robots and relatively uniform task lengths. In doing so, we have also provided an Integer Programming solution, giving the optimal result in many cases.

% There are three natural directions in which to take this work. The first would be finding an optimal polynomial time algorithm for this problem if such an algorithm exists. The second opposing direction would be to search for proof of the hardness of this problem on Path Graphs. Finally, regardless of the first two questions, we would seek to find an explicit approximation bound for the Partition Algorithm, or some alternative polynomial-time algorithm. 

\bibliographystyle{plainurl}% the mandatory bibstyle
\bibliography{bib}

@inbook{lochet2021,
author = {W. Lochet},
title = {A Polynomial Time Algorithm for the <italic>k</italic>-Disjoint Shortest Paths Problem},
booktitle = {Proceedings of the 2021 ACM-SIAM Symposium on Discrete Algorithms (SODA)},
chapter = {},
pages = {169-178},
doi = {10.1137/1.9781611976465.12},
URL = {https://epubs.siam.org/doi/abs/10.1137/1.9781611976465.12},
eprint = {https://epubs.siam.org/doi/pdf/10.1137/1.9781611976465.12}
}

@article{burger2020mobile,
  title={A mobile robotic chemist},
  author={Burger, B. and Maffettone, P. M. and Gusev, V. V. and Aitchison, C. M. and Bai, Y. and Wang, X. and Li, X. and Alston, B. M. and Li, B. and Clowes, R. and others},
  journal={Nature},
  volume={583},
  number={7815},
  pages={237--241},
  year={2020},
  publisher={Nature Publishing Group UK London}
}

@article{king2011rise,
  title={Rise of the robo scientists},
  author={King, R. D.},
  journal={Scientific American},
  volume={304},
  number={1},
  pages={72--77},
  year={2011},
  publisher={JSTOR}
}

@article{granda2018controlling,
  title={Controlling an organic synthesis robot with machine learning to search for new reactivity},
  author={Granda, J. M. and Donina, L. and Dragone, V. and Long, D. and Cronin, L.},
  journal={Nature},
  volume={559},
  number={7714},
  pages={377--381},
  year={2018},
  publisher={Nature Publishing Group UK London}
}

@misc{langner2019ternary,
      title={Beyond Ternary OPV: High-Throughput Experimentation and Self-Driving Laboratories Optimize Multi-Component Systems}, 
      author={S. Langner and F. Häse and J. D. Perea and T. Stubhan and J. Hauch and L. M. Roch and T. Heumueller and A. Aspuru-Guzik and C. J. Brabec},
      year={2019},
      eprint={1909.03511},
      archivePrefix={arXiv},
      primaryClass={physics.app-ph}
}

@misc{macleod2020selfdriving,
      title={Self-driving laboratory for accelerated discovery of thin-film materials}, 
      author={B. P. MacLeod and F. G. L. Parlane and T. D. Morrissey and F. Häse and L. M. Roch and K. E. Dettelbach and R. Moreira and L. P. E. Yunker and M. B. Rooney and J. R. Deeth and V. Lai and G. J. Ng and H. Situ and R. H. Zhang and M. S. Elliott and T. H. Haley and D. J. Dvorak and A. Aspuru-Guzik and J. E. Hein and C. P. Berlinguette},
      year={2020},
      eprint={1906.05398},
      archivePrefix={arXiv},
      primaryClass={physics.app-ph}
}

@article{li2015synthesis,
  title={Synthesis of many different types of organic small molecules using one automated process},
  author={Li, J. and Ballmer, S. G. and Gllis, E. P. and Fujii, S. and Schmidt, M. J. and Palazzolo, A. M. E. and Lehmann, J. W. and Morehouse, G. F. and Burke, M. D.},
  journal={Science},
  volume={347},
  number={6227},
  pages={1221--1226},
  year={2015},
  publisher={American Association for the Advancement of Science}
}

@article{qamar2023trmaxalloc,
  title={TRMaxAlloc: Maximum task allocation using reassignment algorithm in multi-UAV system},
  author={Qamar, R. A. and Sarfraz, M. and Ghauri, S. A. and Mahmood, A.},
  journal={Computer Communications},
  volume={206},
  pages={110--123},
  year={2023},
  publisher={Elsevier}
}

@article{Liu2023,
title = {Balanced task allocation and collision-free scheduling of multi-robot systems in large spacecraft structure manufacturing},
journal = {Robotics and Autonomous Systems},
volume = {159},
pages = {104289},
year = {2023},
issn = {0921-8890},
author = {S. Liu and J. Shen and W. Tian and J. Lin and P. Li and B. Li},
}

@inproceedings{erlebach2022parameterized,
  title={Parameterized Temporal Exploration Problems},
  author={Erlebach, T. and Spooner, J. T.},
  booktitle={1st Symposium on Algorithmic Foundations of Dynamic Networks (SAND 2022)},
  year={2022},
  organization={Schloss Dagstuhl-Leibniz-Zentrum f{\"u}r Informatik}
}

@article{michail2016traveling,
  title={Traveling salesman problems in temporal graphs},
  author={Michail, O. and Spirakis, P. G.},
  journal={Theoretical Computer Science},
  volume={634},
  pages={1--23},
  year={2016},
  publisher={Elsevier}
}

@InProceedings{erlebach_et_al:LIPIcs.ICALP.2019.141,
  author =	{Erlebach, T. and Kammer, F. and Luo, K. and Sajenko, A. and Spooner, J. T.},
  title =	{{Two Moves per Time Step Make a Difference}},
  booktitle =	{46th International Colloquium on Automata, Languages, and Programming (ICALP 2019)},
  pages =	{141:1--141:14},
  series =	{Leibniz International Proceedings in Informatics (LIPIcs)},
  year =	{2019},
  volume =	{132},
  address =	{Dagstuhl, Germany},
}

@article{erlebach2021temporal,
  title={On temporal graph exploration},
  author={Erlebach, T. and Hoffmann, M. and Kammer, F.},
  journal={Journal of Computer and System Sciences},
  volume={119},
  pages={1--18},
  year={2021},
  publisher={Elsevier}
}

@article{arrighi2023kernelizing,
  title={Kernelizing Temporal Exploration Problems},
  author={Arrighi, E. and Fomin, F. V. and Golovach, P. and Wolf, P.},
  journal={arXiv preprint arXiv:2302.10110},
  year={2023}
}

@inproceedings{adamson2022faster,
  title={Faster exploration of some temporal graphs},
  author={Adamson, D. and Gusev, V. V. and Malyshev, D. and Zamaraev, V.},
  booktitle={1st Symposium on Algorithmic Foundations of Dynamic Networks (SAND 2022)},
  year={2022},
  organization={Schloss Dagstuhl-Leibniz-Zentrum f{\"u}r Informatik}
}

@article{akrida2021temporal,
  title={The temporal explorer who returns to the base},
  author={Akrida, E. C. and Mertzios, G. B. and Spirakis, P. G. and Raptopoulos, C.},
  journal={Journal of Computer and System Sciences},
  volume={120},
  pages={179--193},
  year={2021},
  publisher={Elsevier}
}

@article{bumpus2023edge,
  title={Edge exploration of temporal graphs},
  author={Bumpus, B. M. and Meeks, K.},
  journal={Algorithmica},
  volume={85},
  number={3},
  pages={688--716},
  year={2023},
  publisher={Springer}
}

@article{deligkas2022optimizing,
  title={Optimizing reachability sets in temporal graphs by delaying},
  author={Deligkas, A. and Potapov, I.},
  journal={Information and Computation},
  volume={285},
  pages={104890},
  year={2022},
  publisher={Elsevier}
}

@inproceedings{erlebach2018faster,
  title={Faster exploration of degree-bounded temporal graphs},
  author={Erlebach, T. and Spooner, J. T.},
  booktitle={43rd International Symposium on Mathematical Foundations of Computer Science (MFCS 2018)},
  year={2018},
  organization={Schloss Dagstuhl-Leibniz-Zentrum fuer Informatik}
}

@article{bodlaender2019exploring,
  title={On exploring always-connected temporal graphs of small pathwidth},
  author={Bodlaender, H. L. and van der Zanden, T. C.},
  journal={Information Processing Letters},
  volume={142},
  pages={68--71},
  year={2019},
  publisher={Elsevier}
}

@misc{gurobi,
  author = {{Gurobi Optimization, LLC}},
  title = {{Gurobi Optimizer Reference Manual}},
  year = 2024,
  url = "https://www.gurobi.com"
}

@article{erlebach2022exploration,
  title={Exploration of k-edge-deficient temporal graphs},
  author={Erlebach, T. and Spooner, J. T.},
  journal={Acta Informatica},
  volume={59},
  number={4},
  pages={387--407},
  year={2022},
  publisher={Springer}
}

@Article{Itai1982,
  author    = {Itai, Alon and Papadimitriou, Christos H and Szwarcfiter, Jayme Luiz},
  journal   = {SIAM Journal on Computing},
  title     = {Hamilton paths in grid graphs},
  year      = {1982},
  number    = {4},
  pages     = {676--686},
  volume    = {11},
  publisher = {SIAM},
}

@phdthesis{taghian2020exploring,
  title={Exploring temporal cycles and grids},
  author={Taghian Alamouti, S.},
  year={2020},
  school={Concordia University}
}

@article{ourPreviousPaper,
  author       = {D. Adamson and
                  N. Flaherty and
                  I. Potapov and
                  P. G. Spirakis},
  title        = {Collision-free Robot Scheduling},
  journal      = {Inf. Comput.},
  volume       = {304},
  pages        = {105294},
  year         = {2025}
}

@article{Stern_2021, title={Multi-Agent Pathfinding: Definitions, Variants, and Benchmarks}, volume={10}, url={https://ojs.aaai.org/index.php/SOCS/article/view/18510}, DOI={10.1609/socs.v10i1.18510}, number={1}, journal={Proceedings of the International Symposium on Combinatorial Search}, author={Stern, R. and Sturtevant, N. and Felner, A. and Koenig, S. and Ma, H. and Walker, T. and Li, J. and Atzmon, D. and Cohen, L. and Kumar, T. K. and Barták, R. and Boyarski, E.}, year={2021}, month={Sep.}, pages={151-158} }

@inproceedings{standley2010finding,
  title={Finding optimal solutions to cooperative pathfinding problems},
  author={Standley, T.},
  booktitle={Proceedings of the AAAI conference on artificial intelligence},
  volume={24},
  number={1},
  pages={173--178},
  year={2010}
}

@article{SHARON201540,
title = {Conflict-based search for optimal multi-agent pathfinding},
journal = {Artificial Intelligence},
volume = {219},
pages = {40-66},
year = {2015},
issn = {0004-3702},
doi = {https://doi.org/10.1016/j.artint.2014.11.006},
url = {https://www.sciencedirect.com/science/article/pii/S0004370214001386},
author = {G. Sharon and R. Stern and A. Felner and N. R. Sturtevant}
}

@inproceedings{liu2019task,
  title={Task and path planning for multi-agent pickup and delivery},
  author={Liu, M. and Ma, H. and Li, J. and Koenig, S.},
  booktitle={Proceedings of the International Joint Conference on Autonomous Agents and Multiagent Systems (AAMAS)},
  year={2019}
}

@article{dantzig2016linear,
  title={Linear programming and extensions},
  author={Dantzig, G. B.},
  year={2016},
  publisher={Princeton university press}
}

\appendix

\newpage

\section{Algorithms implementation details}\label{app:implementation}
The code used for the experiments is available at \url{https://github.com/sea26-robots/code}. Datasets analogous to $DS1$–$DS5$ can be generated on the fly by following the code documentation.

For the IP implementation, we used the linear programming solver Gurobi \cite{gurobi}, accessed via the gurobipy library for Python. To accelerate the work of IP, two modifications compared to the algorithm's description in the paper were made. First, the upper bound $\tau$ on the time of a schedule was reduced to the time of the set of schedules provided by PA. As PA works for the fraction of the time required by IP, its usage prior IP does not contribute to the time of work of IP. The second modification was to remove constraints (\ref{eq:no_crossover}) as it would lead to a set of schedules with possible collisions caused by traversing only, for which the collision-free set of schedules of the same length exists according to Lemma \ref{lem:traverse}.

Regarding PA, we made some modifications compared to the description in \cite{ourPreviousPaper} to prevent collisions with stationary robots that may remain without tasks. Specifically, the algorithm formulation in \cite{ourPreviousPaper} suggests using a partition of the tasks to allocate each part to a robot. However, when considering a large number of robots, situations in which one or more robots become redundant and remain without tasks become more common, and potential collisions between moving and stationary robots may occur. To prevent such collisions, we slightly updated the algorithm.

\textbf{Update 1}. First, on the step when we choose the partition for the robot $R_j$ given the subset of tasks $[t_1,\dots,t_i]$, among different partitions providing the sets of schedules of the same timespan, we choose the one that assigns the same task to the nearest robots. For instance, if the partitions $S'$ and $S''$ differ in the assignment of the task $t_i$ to the robot $R_{j-1}$ in $S'$ or $R_{j}$ in $S''$, such that $\dist(t_i, R_j) < \dist(t_i, R_{j-1})$, and result in the sets of schedules of the same duration, we choose the partition $S''$. It prevents a possible collision that would occur if $R_{j}$ was stationary and on the way from $R_{j-1}$ to $t_i$. 

\textbf{Update 2}. Second, given the robot $R_j$, we forbid it to consider partitions that start on the first $j-1$ or finish on the last $k-j$ vertices as such partitions would lead to collisions and would be rejected anyway. 

\textbf{Update 3}. Finally, we use the table $SC$ that stores the set of schedules for the given subsets of tasks and robots. $SC[j, i+1]$ contains the set of schedules for the robots $R_1, \dots, R_{j}$ and the tasks $[t_1,\dots,t_i]$, while $SC[j, 1]$ contains an empty set of schedules. We use the table of sets of schedules $SC$ instead of the table of splitting points $S$ used in \cite{ourPreviousPaper} to construct the schedules directly from the partition.

The pseudocode of PA with the modifications is provided in Algorithm \ref{alg:pa}. The function $len$ returns the duration of the schedule or the set of schedules.

\section{Scheduling problem examples}\label{app:example}

In this section, we consider two examples of a scheduling problem for 2 robots on the path graph with 12 vertices taken from $DS1$ of the experiments. The first example has 8 tasks of different durations and is shown on Figure \ref{fig:example}. For this example, we present sets of schedules provided by the compared algorithms in Table \ref{tab:example}. All algorithms find the solutions of different lengths with GA being the worst (27 steps), IP being the best (18 steps), and PA being just one step longer than IP. The second example of a scheduling problem is shown in Figure~\ref{fig:example2} and presents the instance that was the most time-consuming for IP to solve among all instances in $DS1$.

\section{Hardware usage}

The experiments were conducted on a Linux cluster node equipped with two 84-core AMD EPYC $9,634$ CPUs (3.7 GHz) and 1.5 TB of memory (9 GB per core). For each experiment, a single CPU core was used, while the remaining cores remained idle.

\section{The number of tasks, the number of robots, and the maximum task duration analysis for longer paths}\label{sec:stick_100}

We stick to the maximum path graph order $n=100$ and show that the number of tasks, the number of robots, and the maximum task duration affect the performance of the algorithms similarly as in the experiment with $DS1$ with PA consistently outperforming both GA and RA. Figure \ref{fig:ave_sch_100} (for $DS2$-$DS4$) and \ref{fig:ave_sch_100_DS5} (for $DS5$) show that for all algorithms the average schedule timespan increases with increasing number of tasks or $d_{max}$ and decreasing number of robots.
In all cases, PA demonstrates clear outperformance over GA and RA, except the case with the smallest number of tasks, in which algorithms show comparable performance.

\begin{figure}
  \centering
  \begin{subfigure}{0.49\linewidth}
\includegraphics[width=1\linewidth]{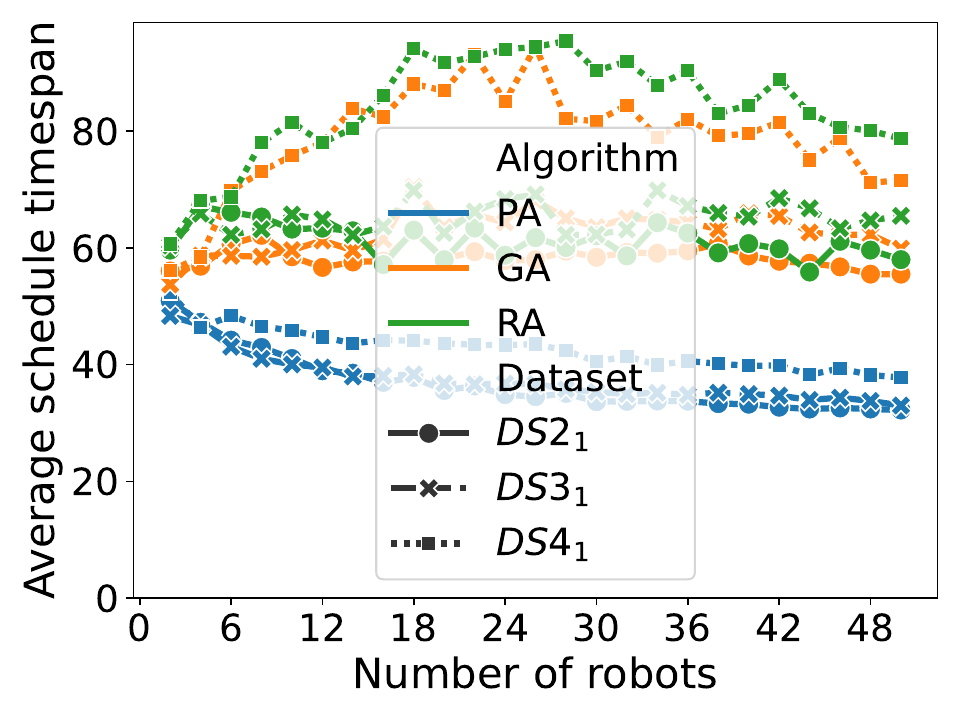}
  \end{subfigure}
  \begin{subfigure}{0.49\linewidth}
\includegraphics[width=1\linewidth]{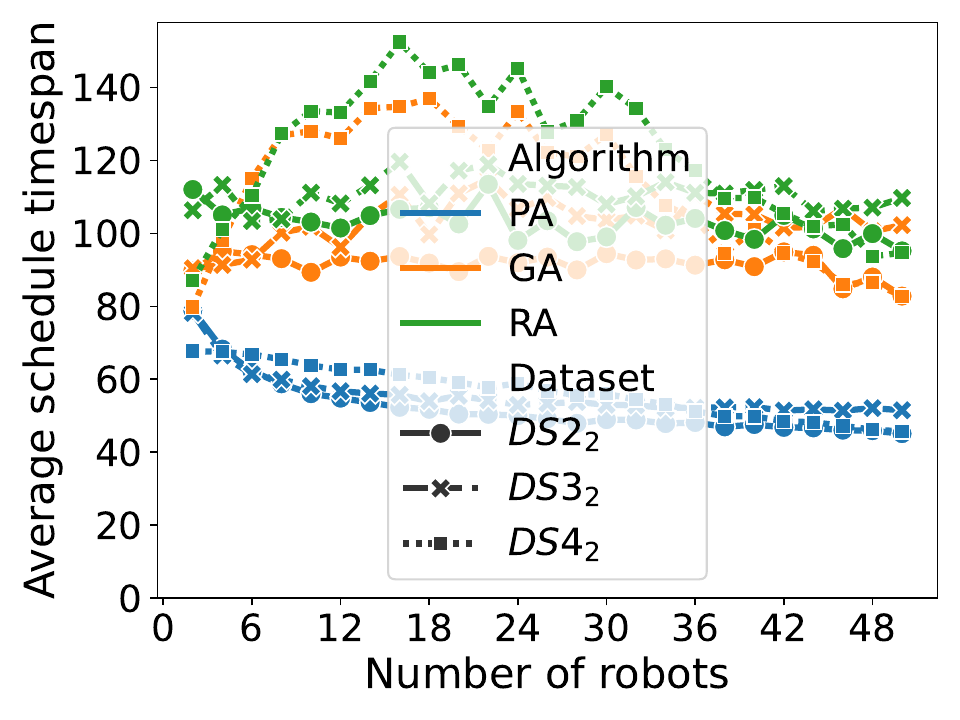}
  \end{subfigure}
  \begin{subfigure}{0.49\linewidth}
\includegraphics[width=1\linewidth]{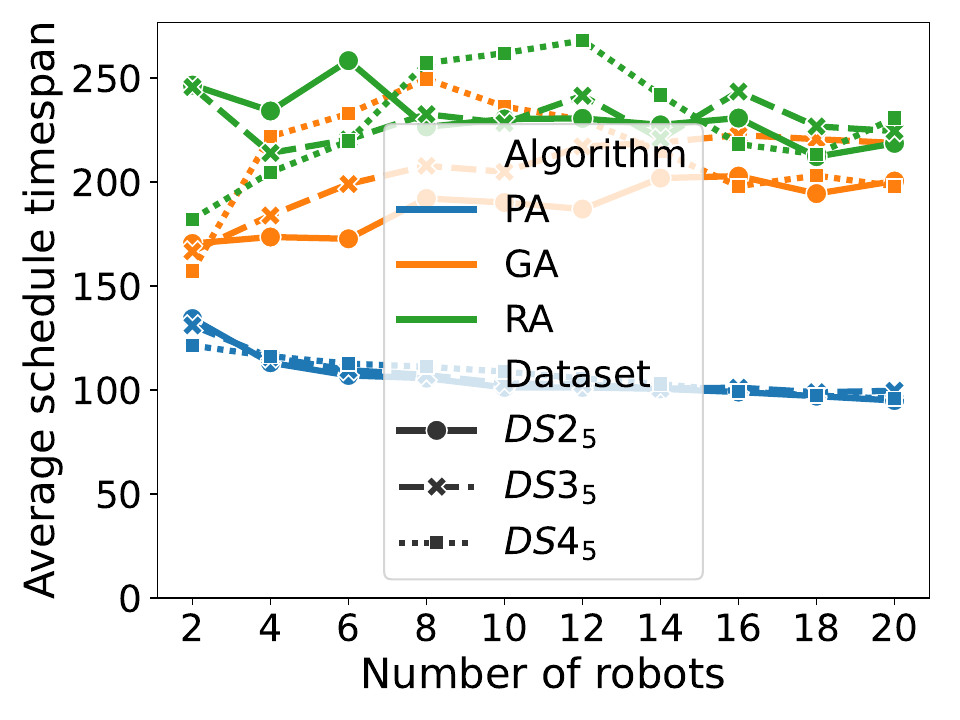}
  \end{subfigure}
  \begin{subfigure}{0.49\linewidth}
\includegraphics[width=1\linewidth]{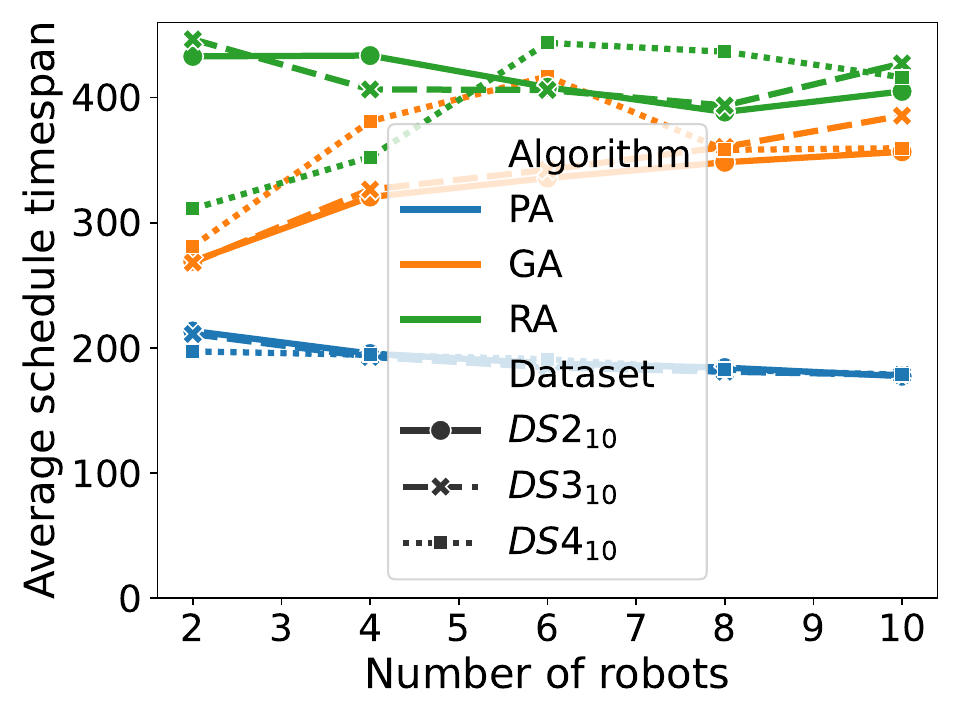}
  \end{subfigure}
  \caption{Experiments with the subsets with different task-to-robot ratios: 1 (top left), 2 (top right), 5 (bottom left), and 10 (bottom right).}
  \label{fig:fixed_ratio} 
\end{figure}

\section{Fixed task-to-robots ratios on the datasets $DS2$-$DS4$}\label{app:fixed_ratio}

The analysis of the subsets of $DS2$–$DS4$ with a fixed task-to-robot ratio reveals a consistent difference in the performance of the algorithms. Figure \ref{fig:fixed_ratio} shows that the average schedule timespan of PA is lower than that of GA and RA for all numbers of robots and tasks with a fixed task-to-robot ratio, and this difference increases as the numbers of tasks and robots grow. The average schedule timespan of PA decreases with the increasing number of robots (and tasks), whereas for GA and RA, the average schedule timespan either decreases more slowly (RA for $DS_{10}$) or even increases (GA for $DS_{10}$).

\begin{algorithm}
\caption{Partition Scheduling Algorithm (\cite{ourPreviousPaper}) with our updates}
\scriptsize
\begin{algorithmic}[1]
% \Procedure{MyProcedure}{}
    \State Initialize the empty table $SC$
    \For{$j$ from $1$ to $k$}
        \State $SC[j, 1] \gets \emptyset$
    \EndFor
    \textcolor{red}{\For{$i$ from $1$ to $m$}
        \State $SC[1, i] \gets C([t_1,\dots,t_i], sv_0)$\Comment{\textbf{Update 3}}
    \EndFor}
    \For{$j$ from $2$ to $k$} 
        \For{$l$ from $1$ to $m$} 
            \State $min\_len \gets +infinity$
            \State $min\_i \gets 1$
            \For{$i$ from $1$ to $l$} 
                \textcolor{red}{
                \If{($v_i \leq j-1$) or \\ 
                \hspace{5.5em}($v_i > n - k+j$)} \Comment{\textbf{Update 2}}
                    \State \textbf{continue}
                \EndIf}
                \State $cur\_len \gets max(len(SC[j-1,i]),$ \\
                \hspace{11em}$len(C([t_{i+1},\dots,t_{l}], j)))$
                \If{$cur\_len < min\_len$}
                    \State $min\_i \gets i$
                    \State $min\_len \gets cur\_len$
                \textcolor{red}{\ElsIf{($cur\_len = min\_len$) and \\
                \hspace{7.5em}($\dist(sv_j, v_i) > \dist(sv_{j-1}, v_i)$)}\Comment{\textbf{Update 1}}
                    \State $min\_i\gets i$}
                \EndIf
            \EndFor{}
            \State $SC[j,l+1] \gets SC[j-1, min\_i],$ \\
            \hspace{9.5em}$C([t_{min\_i+1},\dots,t_{l}], j))$
        \EndFor{}
    \EndFor
    \State \textbf{return} $SC[k,m+1]$
% \EndProcedure
\end{algorithmic}
\label{alg:pa}
\end{algorithm}

\begin{figure}
    \centering
    \begin{tikzpicture}
    \scriptsize
    \tikzset{vertex/.style = {shape=circle,draw,minimum size=2.5em}}
    \tikzset{edge/.style = {-}}
    \tikzset{arrow/.style = {-}}
    \tikzset{robot/.style = {
    circle,
    draw,
    minimum size=4.5em,
    inner sep=-0.5pt,
   }}
    \node[robot](r_1) at (10, 6.7) {\includegraphics[width=4.5em]{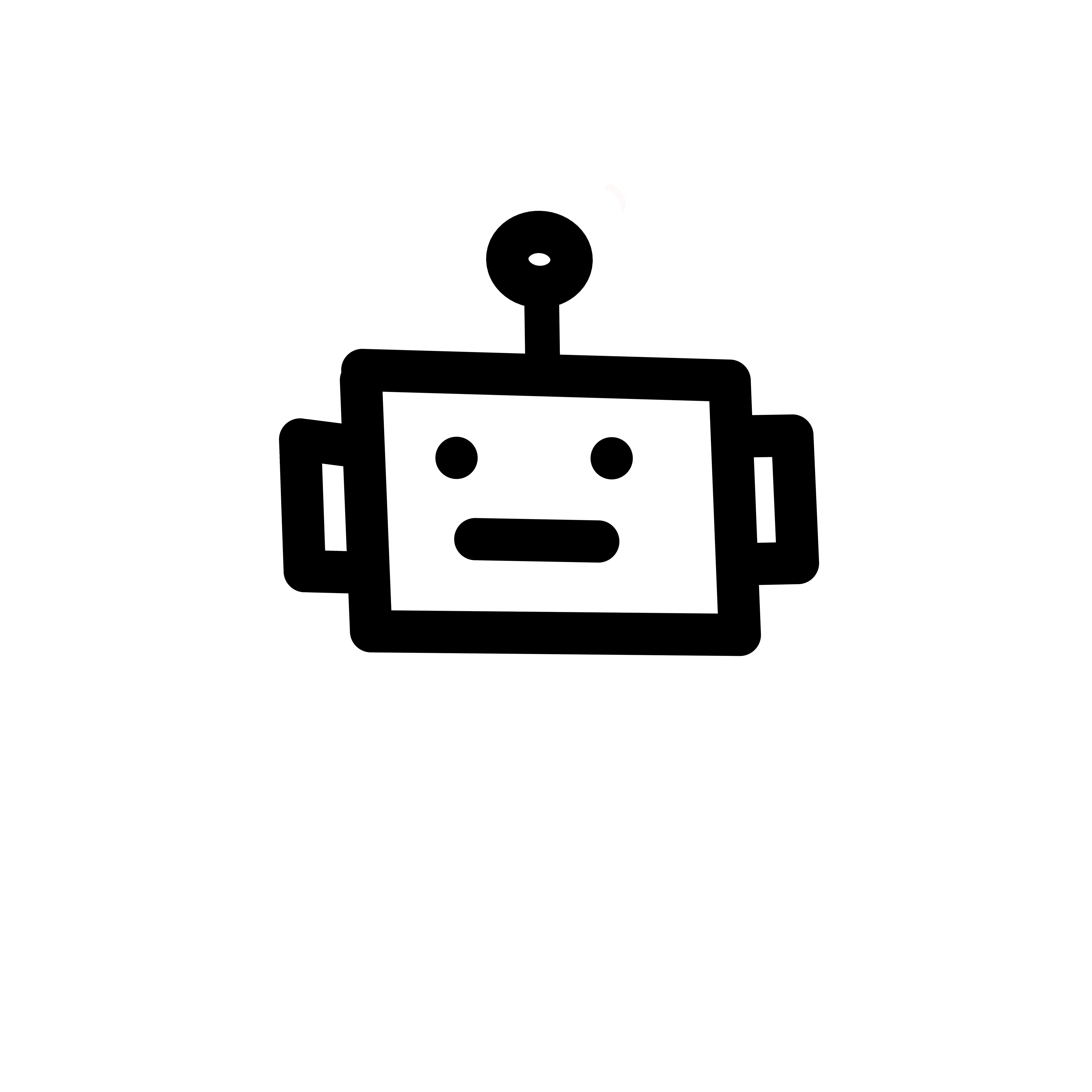}};
    \node[robot](r_2) at (11, 6.7) {\includegraphics[width=4.5em]{figures/robot_image.png}};
    \node[vertex,draw=red](v_4) at (3,6){$v_4,{\color{red}3}$};
    \node[vertex, draw=red](v_1) at (0,6){$v_1,{\color{red}3}$};
    \node[vertex, draw=red](v_5) at (4,6){$v_5,{\color{red}3}$};
    \node[vertex, draw=red](v_6) at (5,6){$v_6,{\color{red}1}$};
    \node[vertex, draw=red](v_7) at (6,6){$v_7,{\color{red}3}$};
    \node[vertex, draw=red](v_8) at (7,6){$v_8,{\color{red}1}$};
    \node[vertex, draw=red](v_9) at (8,6){$v_9,{\color{red}3}$};
    \node[vertex](v_10) at (9,6){$v_{10}$};
    \node[vertex](v_11) at (10,6){$v_{11}$};
    \node[vertex, draw=red](v_12) at (11,6){$v_{12},{\color{red}1}$};
    \node[vertex](v_2) at (1,6) {$v_2$};
    \node[vertex](v_3) at (2,6) {$v_3$};
    \draw[edge] (v_1) to (v_2);
    \draw[edge] (v_2) to (v_3);
    \draw[edge] (v_3) to (v_4);
    \draw[edge] (v_4) to (v_5);
    \draw[edge] (v_5) to (v_6);
    \draw[edge] (v_6) to (v_7);
    \draw[edge] (v_7) to (v_8);
    \draw[edge] (v_8) to (v_9);
    \draw[edge] (v_9) to (v_10);
    \draw[edge] (v_10) to (v_11);
    \draw[edge] (v_11) to (v_12);

    \end{tikzpicture}
      \caption{The example presents a path graph with vertices $\{v_1,\dots,v_{12}\}$ and 8 tasks. Circles indicate vertices with assigned tasks, where the duration of each task is shown in red. Two robots have starting positions $v_{11}$ and $v_{12}$.}
      \label{fig:example}
\end{figure}

\begin{figure}
    \centering
    \begin{tikzpicture}
    \scriptsize
    \tikzset{vertex/.style = {shape=circle,draw,minimum size=3em,
    font=\tiny}}
    \tikzset{edge/.style = {-}}
    \tikzset{arrow/.style = {-}}
    \tikzset{robot/.style = {
    circle,
    draw,
    minimum size=4.5em,
    inner sep=-0.5pt,
   }}
    \node[robot](r_1) at (3, 6.7) {\includegraphics[width=4.5em]{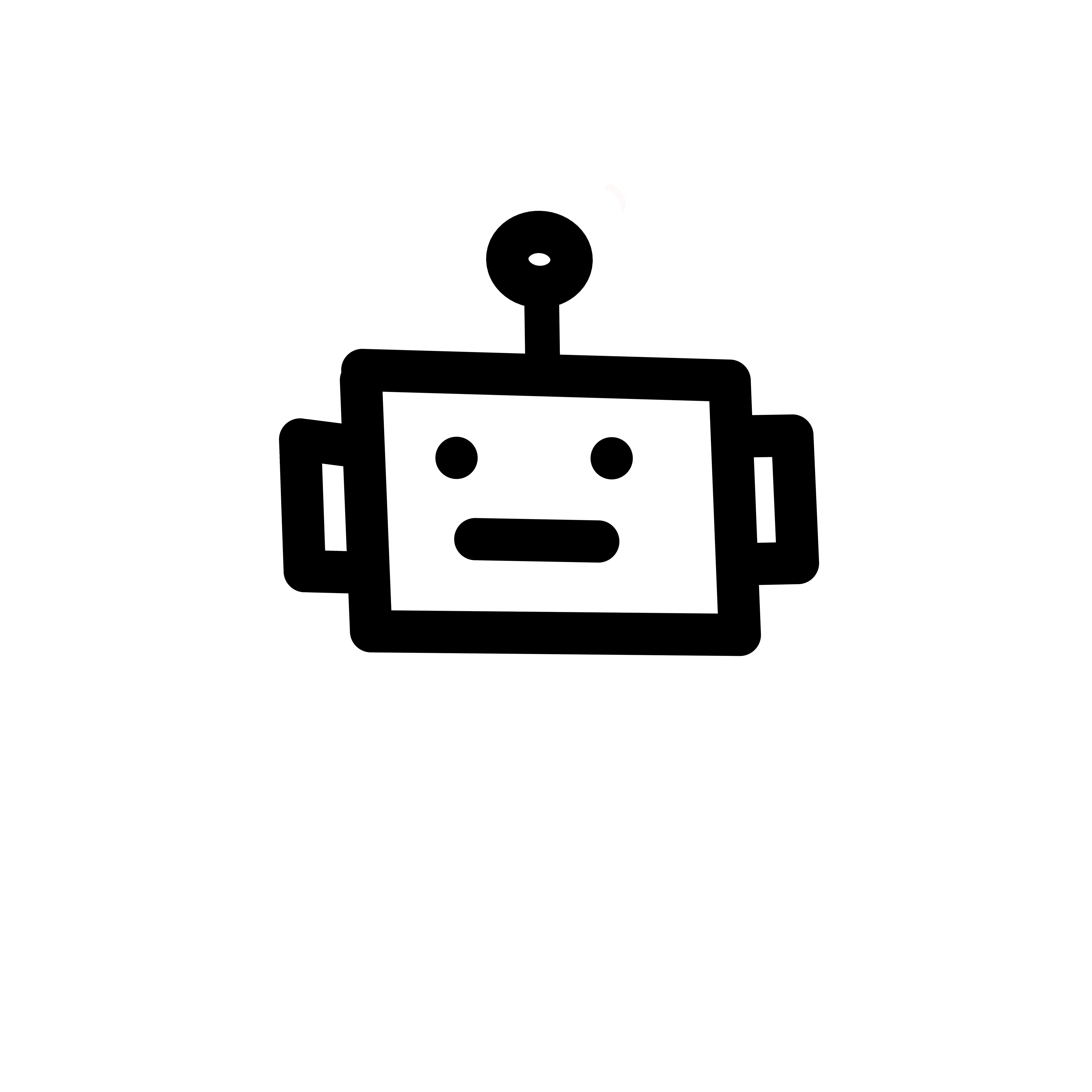}};
    \node[robot](r_2) at (11, 6.7) {\includegraphics[width=4.5em]{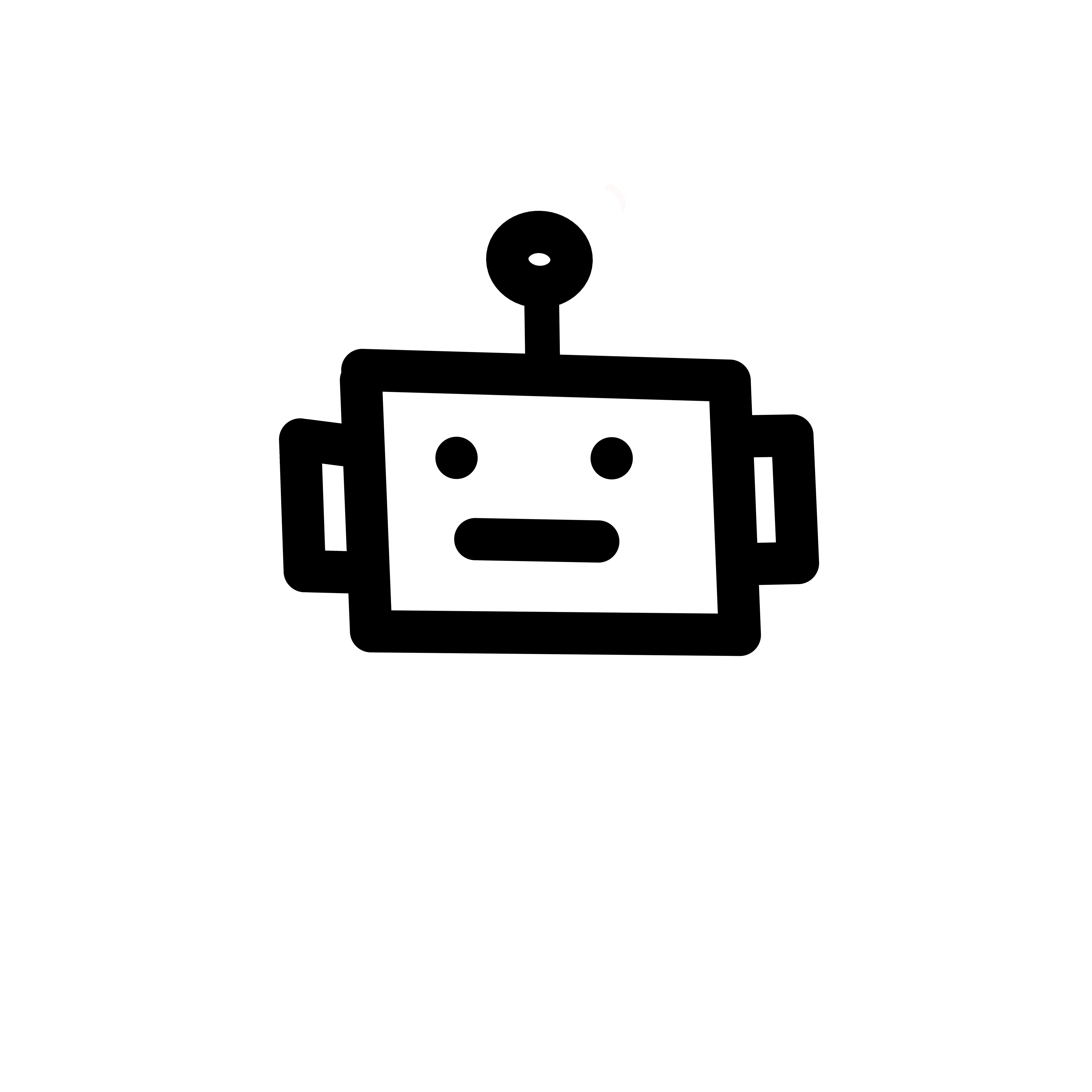}};
    \node[vertex,draw=red](v_4) at (3,6){$v_4$,${\color{red}10}$};
    \node[vertex, draw=red](v_1) at (0,6){$v_1$,${\color{red}5}$};
    \node[vertex, draw=red](v_5) at (4,6){$v_5$,${\color{red}11}$};
    \node[vertex, draw=red](v_6) at (5,6){$v_6$,${\color{red}7}$};
    \node[vertex, draw=red](v_7) at (6,6){$v_7$,${\color{red}11}$};
    \node[vertex, draw=red](v_8) at (7,6){$v_8$,${\color{red}11}$};
    \node[vertex, draw=red](v_9) at (8,6){$v_9$,${\color{red}13}$};
    \node[vertex, draw=red](v_10) at (9,6){$v_{10}$,${\color{red}10}$};
    \node[vertex, draw=red](v_11) at (10,6){$v_{11}$,${\color{red}11}$};
    \node[vertex, draw=red](v_12) at (11,6){$v_{12}$,${\color{red}5}$};
    \node[vertex, draw=red](v_2) at (1,6) {$v_2$,${\color{red}3}$};
    \node[vertex, draw=red](v_3) at (2,6) {$v_3$,${\color{red}{12}}$};
    \draw[edge] (v_1) to (v_2);
    \draw[edge] (v_2) to (v_3);
    \draw[edge] (v_3) to (v_4);
    \draw[edge] (v_4) to (v_5);
    \draw[edge] (v_5) to (v_6);
    \draw[edge] (v_6) to (v_7);
    \draw[edge] (v_7) to (v_8);
    \draw[edge] (v_8) to (v_9);
    \draw[edge] (v_9) to (v_10);
    \draw[edge] (v_10) to (v_11);
    \draw[edge] (v_11) to (v_12);

    \end{tikzpicture}
      \caption{The example presents a path graph with vertices $\{v_1,\dots,v_{12}\}$ and a task on each vertex. The duration of each task is in red. Two robots have starting positions $v_{4}$ and $v_{12}$.}
      \label{fig:example2}
\end{figure}

\begin{algorithm}
\caption{Greedy Scheduling Algorithm}
\scriptsize
\begin{algorithmic}[1]
% \Procedure{MyProcedure}{}
    \State Initialize schedules $C_1,\dots,C_k$ by the robots starting positions.
    \State Initialize the list $\pairs$ of all possible pairs ($t_i$, $R_j$) for all tasks $t_i$ and all robots $R_j$.
    \While{$\pairs \neq \emptyset$}
        \State Sort $\pairs$ by the sum of $d_i$ and the distance between $t_i$ and the latest position of $R_j$ in increasing order
        \For{($t_i$, $R_j$) in $\pairs$} 
            \State Extend $C_j$ to $C_j'$ by adding the task $t_i$ to the end of $C_j$.
            \If{($C_1,\dots, C_{j-1},C_j',C_{j+1},\dots, C_k$) is collision-free}
                \State $C_j \gets C_j'$
                \State Remove all pairs with $t_i$ from $\pairs$
                \State \textbf{break}
            \EndIf
        \EndFor{}
    \EndWhile
    \State \textbf{return} ($C_1, \dots, C_k$)
% \EndProcedure
\end{algorithmic}
\label{ref:alg_greedy}
\end{algorithm}

\begin{algorithm}
\caption{Randomised Scheduling Algorithm}
\scriptsize
\begin{algorithmic}[1] % The [1] adds line numbers
% \Procedure{MyProcedure}{}
    \State Initialize schedules $C_1,\dots,C_k$ by the robots starting positions.
    \State Initialize the list $\pairs$ of all possible pairs ($t_i$, $R_j$) for all tasks $t_i$ and all robots $R_j$.
    \While{$\pairs \neq \emptyset$}
        \State Randomly reshuffle $\pairs$
        \For{($t_i$, $R_j$) in $\pairs$} 
            \State Extend $C_j$ to $C_j'$ by adding the task $t_i$ to the end of $C_j$.
            \If{($C_1,\dots, C_{j-1},C_j',C_{j+1},\dots, C_k$) is collision-free}
                \State $C_j \gets C_j'$
                \State Remove all pairs with $t_i$ from $\pairs$
                \State \textbf{break}
            \EndIf
        \EndFor{}
    \EndWhile
    \State \textbf{return} ($C_1, \dots, C_k$)
% \EndProcedure
\end{algorithmic}
\label{ref:alg_random}
\end{algorithm}

\begin{table}[h]
    \setlength{\tabcolsep}{2.3pt} 
    \centering
    \scriptsize
    \newcommand{\task}[1]{\textbf{\textcolor{red}{#1}}}
    \caption{Sets of schedules for the robots in the example on Figure \ref{fig:example}. A trajectory is represented as a sequence of vertex indices, reflecting the robot’s position at a given timestep. Numbers shown in red indicate that the robot performs a task at that timestep.}
    \label{tab:example}
    \begin{tabular}{c|c|c c c c c c c c c c c c c c c c c c c c c c c c c c c c}
    \hline
   \multicolumn{1}{c|}{Algorithm} & Robot & \multicolumn{28}{c}{Schedule} \rule{0pt}{2.8ex}\\
    \hline
    \multicolumn{1}{c|}{IP} & \multicolumn{1}{c|}{R1} & 11& 10& 9& 8& 7& 6& \task{6}& 5& 4& \task{4}& \task{4}& \task{4}& 4& 3& 2& 1& \task{1}& \task{1}& \task{1}\rule{0pt}{2.8ex}\\
     & R2 & 12& \task{12}& 11& 10& 9& \task{9}& \task{9}& \task{9}& 8& \task{8}& 7& \task{7}& \task{7}& \task{7}& 6& 5& \task{5}& \task{5}& \task{5}\\
    \hline
    PA & R1 & 11& 10& 9& 8& 7& 6& 5& 4& \task{4}& \task{4}& \task{4}& 3& 2& 1& \task{1}& \task{1}& \task{1}\rule{0pt}{2.8ex}\\
     & R2 & 12& \task{12}& 11& 10& 9& \task{9}& \task{9}& \task{9}& 8& \task{8}& 7& \task{7}& \task{7}& \task{7}& 6& \task{6}& 5& \task{5}& \task{5}& \task{5}\\
    \hline
    GA & R1 & 11& 10& 9& \task{9}& \task{9}& \task{9}& 8& \task{8}& 7& \task{7}& \task{7}& \task{7}& 6& \task{6}& 5& \task{5}& \task{5}& \task{5}& 4& \task{4}& \task{4}& \task{4}& 3& 2& 1& \task{1}& \task{1}& \task{1} \rule{0pt}{2.8ex}\\
     & R2 & 12 & \task{12}\\
    \hline
    RA & R1 & 11& 10& 9& 8& \task{8}& 7& 6& 5& 3& \task{4}& \task{4}& \task{4}& 5& \task{5}& \task{5}& \task{5}& 4& 3& 2& 1& \task{1}& \task{1}& \task{1} \rule{0pt}{2.8ex}\\
     & R2 & 12& \task{12}& 11& 10& 9& \task{9}& \task{9}& \task{9}& 8& 7& \task{7}& \task{7}& \task{7}& 6& \task{6} \\
     \hline
     \hline
    \multicolumn{2}{c|}{Timestep}& 0& 1& 2& 3& 4& 5& 6& 7& 8& 9& 10& 11& 12& 13& 14& 15 & 16 & 17 & 18 & 19 & 20 & 21 & 22 & 23 & 24 & 25 & 26 & 27 \rule{0pt}{2.8ex}\\
    \hline
    \end{tabular}
\end{table}

\begin{figure}
    \centering
 \includegraphics[width=0.5\linewidth]{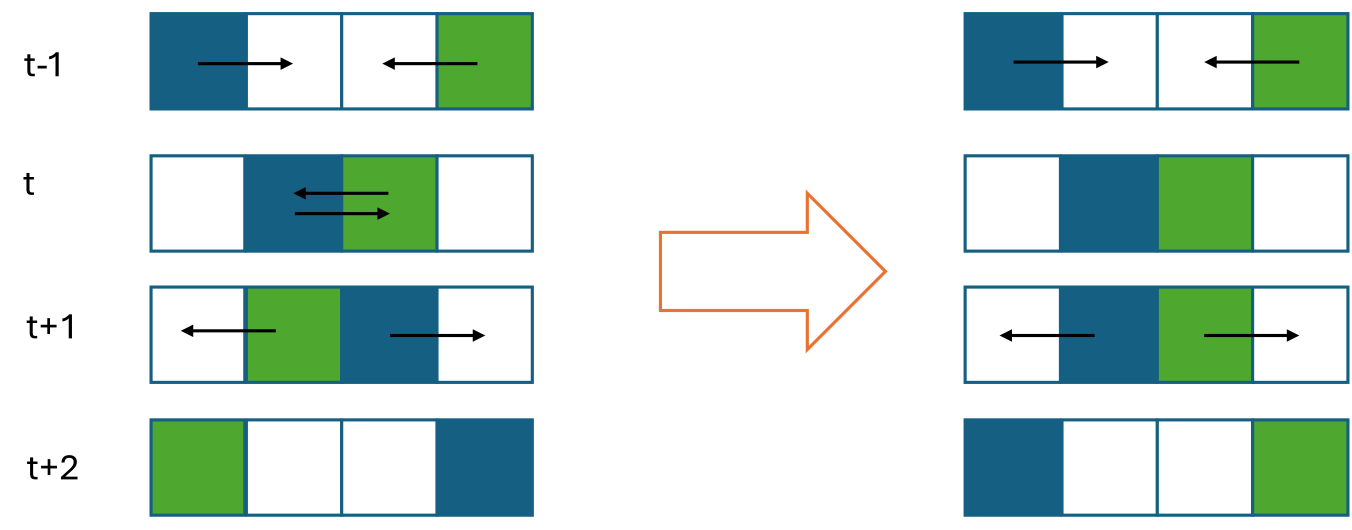}
    \caption{An example of the robots' schedules update to avoid traversing the same edge, the robots are coloured green and blue.}
    \label{fig:traverse}
\end{figure}

\begin{table}[h]
    \centering
    \scriptsize
    \caption{The distributions for tasks and robots allocations as well as the tasks durations for the datasets used in the experiments.}
    \label{tab:distr}
    \begin{tabular}{c|c c c c c}
    \toprule
    Distribution& $DS1$ & $DS2$& $DS3$ & $DS4$ & $DS5$\\
    \midrule
    Tasks durations & Uniform & Uniform &Uneven uniform & Uniform & Uniform \\
    Tasks allocations & Uniform & Uniform&Uniform & Normal & Uniform \\
    Robots allocations & Uniform & Uniform&Uniform & Uniform & Normal\\
    \bottomrule
    \end{tabular}
\end{table}

\begin{figure*}
  \centering
  \begin{subfigure}{0.32\linewidth}
\includegraphics[width=1\linewidth]{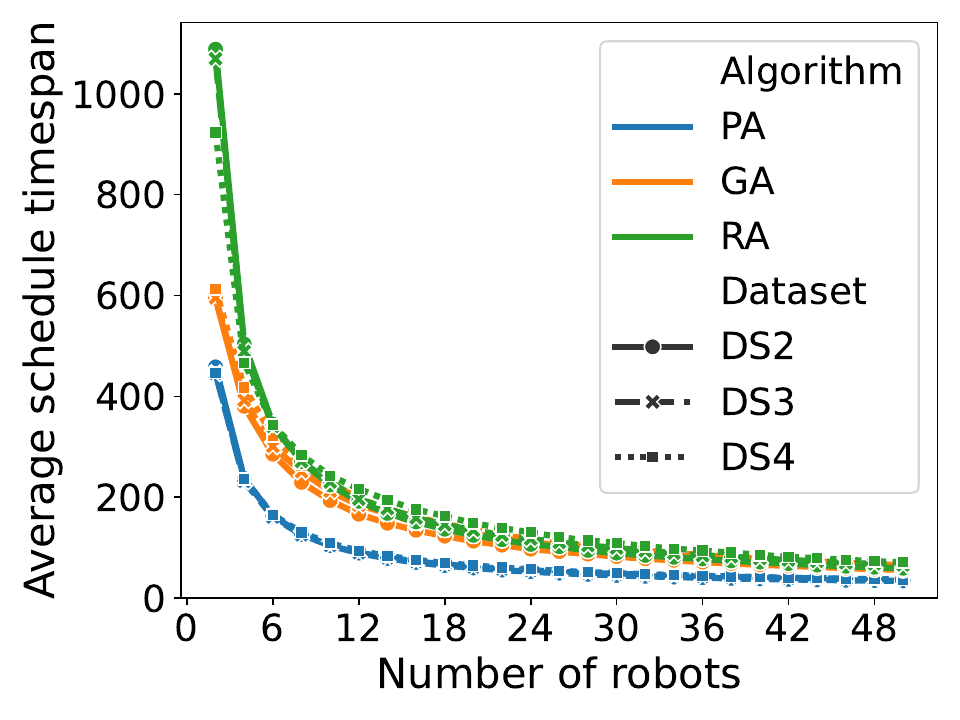}
  \end{subfigure}
  \begin{subfigure}{0.32\linewidth}
\includegraphics[width=1\linewidth]{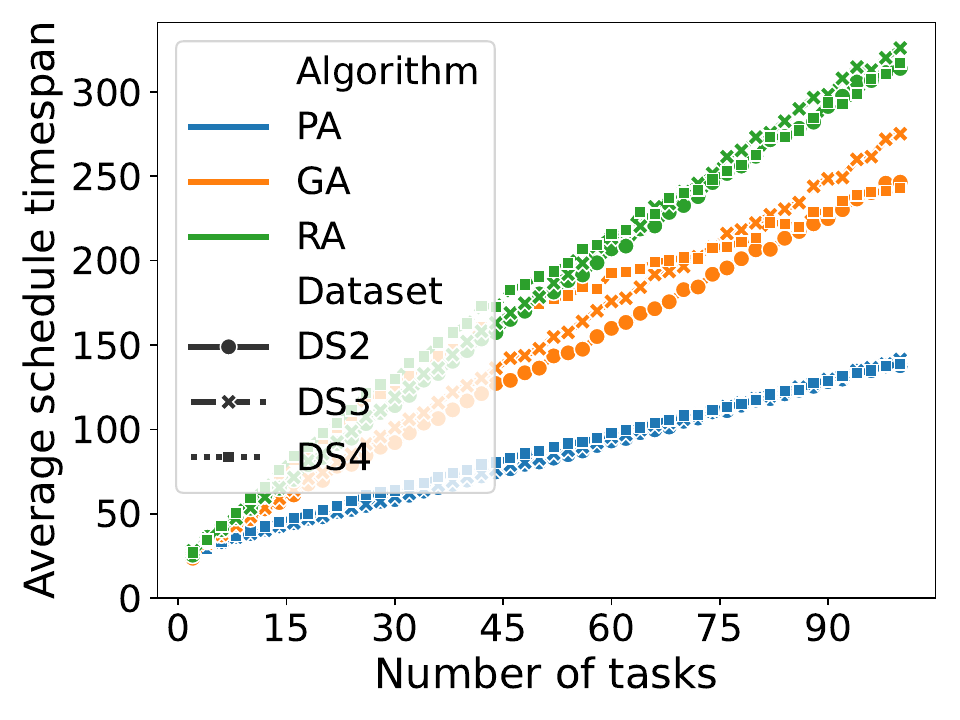}
  \end{subfigure}
  \begin{subfigure}{0.32\linewidth}
\includegraphics[width=1\linewidth]{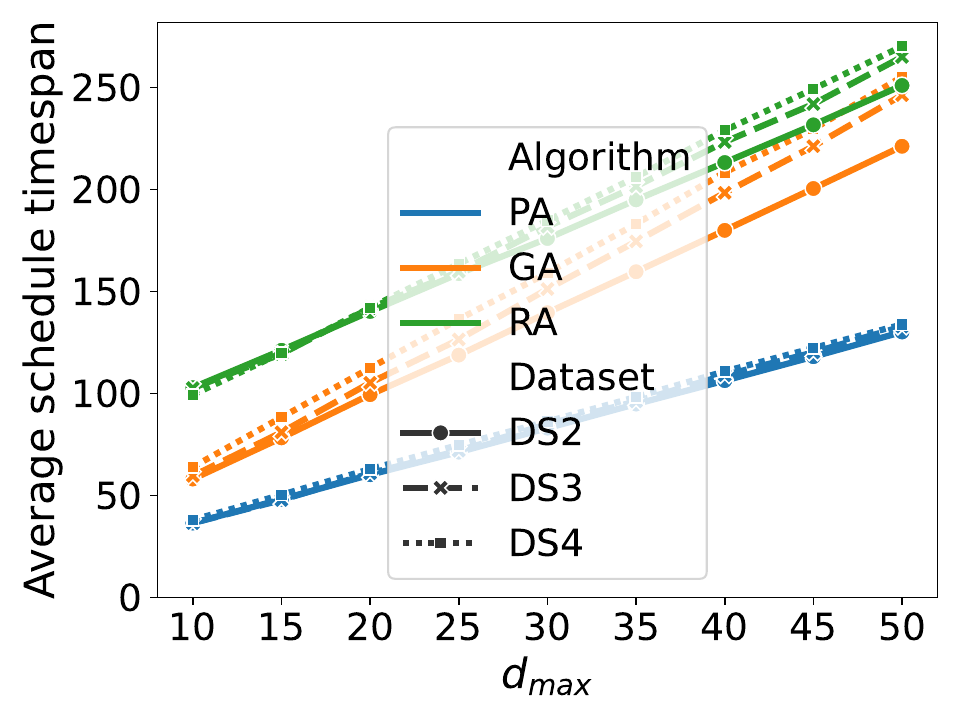}
  \end{subfigure}
  \caption{The average schedule timespan for $n=100$ for the different number of robots (left), number of tasks (middle), and maximum task duration (right) for $DS2$-$DS4$.}
  \label{fig:ave_sch_100} 
\end{figure*}

\begin{figure*}
  \centering
  \begin{subfigure}{0.32\linewidth}
\includegraphics[width=1\linewidth]{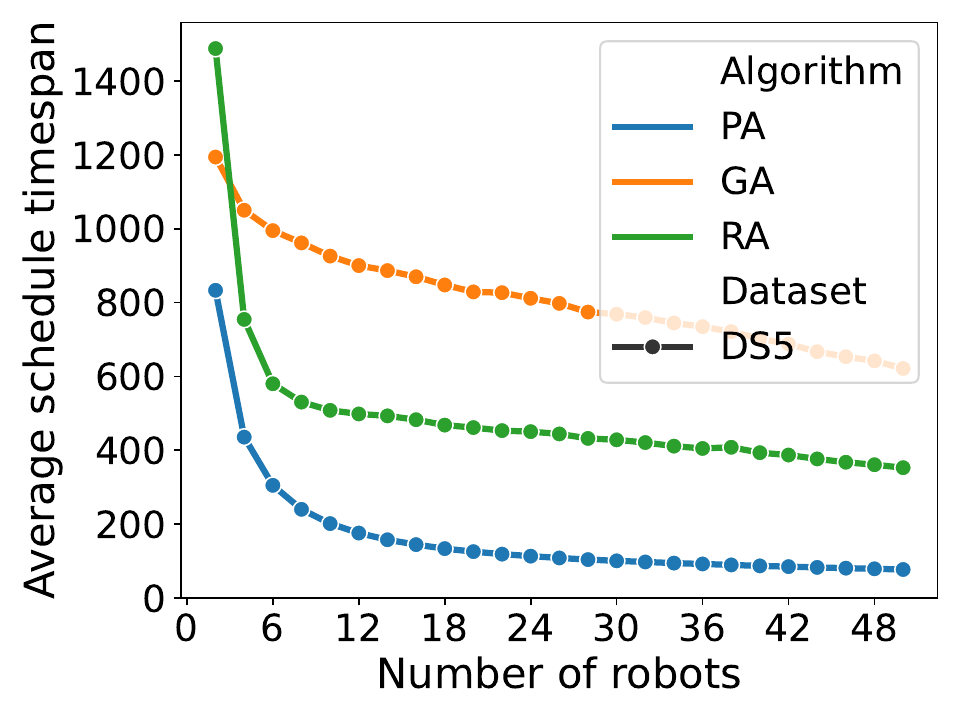}
  \end{subfigure}
  \begin{subfigure}{0.32\linewidth}
\includegraphics[width=1\linewidth]{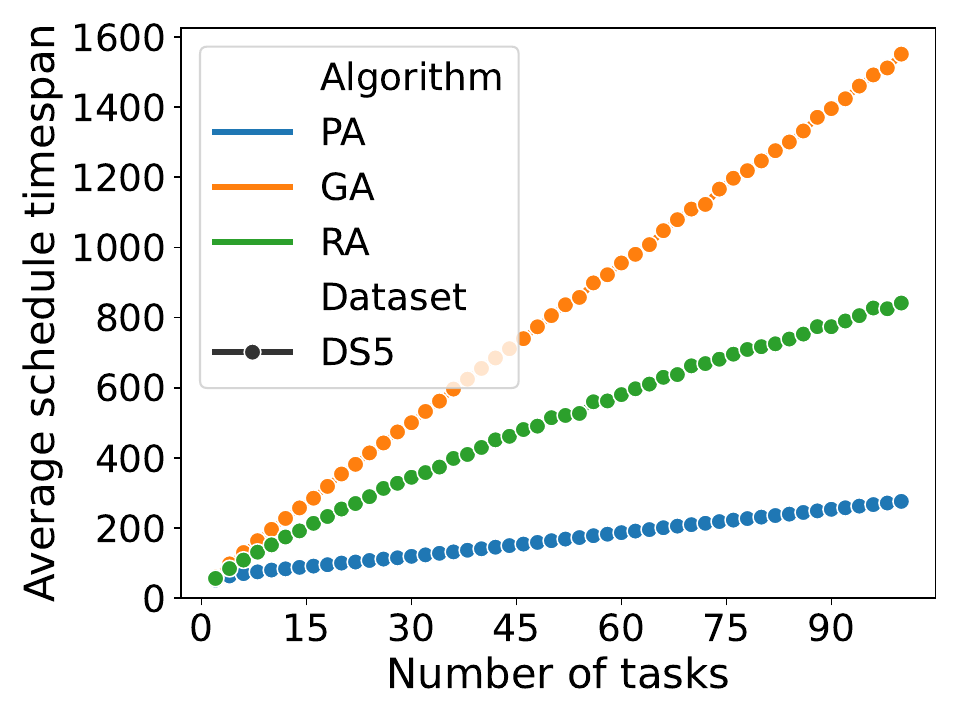}
  \end{subfigure}
  \begin{subfigure}{0.32\linewidth}
\includegraphics[width=1\linewidth]{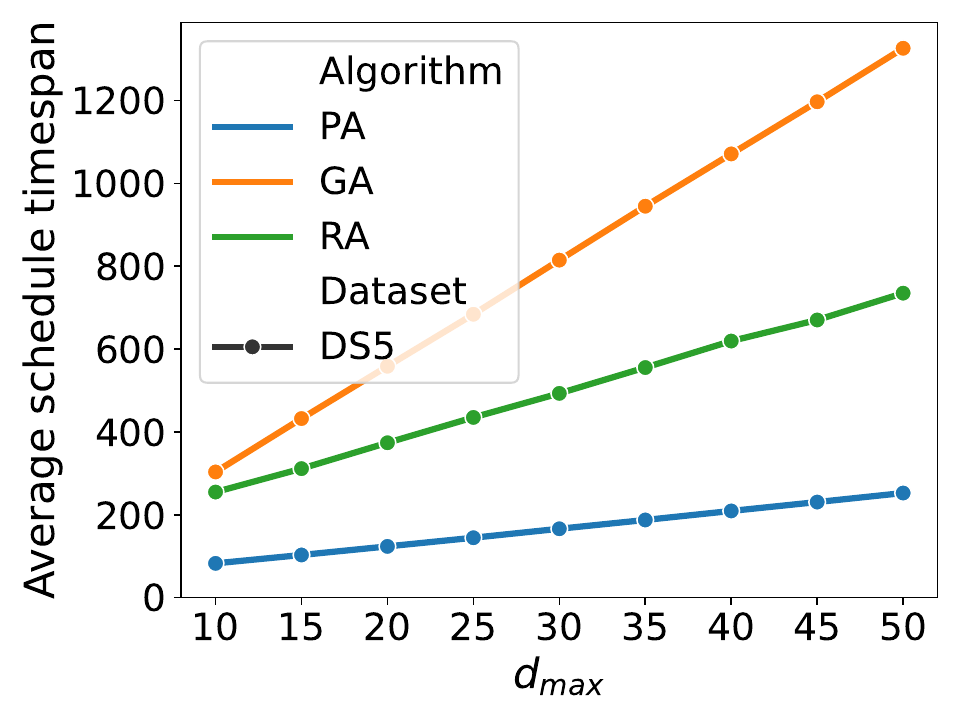}
  \end{subfigure}
  \caption{The average schedule timespan for $n=100$ for the different number of robots (left), number of tasks (middle), and maximum task duration (right) for $DS5$.}
  \label{fig:ave_sch_100_DS5} 
\end{figure*}

% \section{Greedy Algorithm for Grids}
% Given an instance of robot scheduling on a grid graph we can use the following greedy algorithm to arrive at a schedule. 
% \begin{itemize}
%     \item For each robot, work out its closest accessible task (closest can either mean minimum distance away OR minimum distance + task duration - accessible meaning one it can get to) 
%  \item If any robot's closest task is the same as another's then the closest robot to the task completes it.
%  \item If multiple are the same distance away from a task then randomly choose one. 
% \item  recompute the other robot(s)' closest task and resolve ties again. 

% \end{itemize}

% \section{Extended experiments}

% We use the following notations for the compared algorithms: \emph{PA} for the partition algorithm, \emph{GA} for the greedy algorithm, \emph{IP} for the linear integer programming algorithm, and \emph{RA} for the random algorithm.

% The efficiency of a scheduling algorithm can be measured using various metrics.
% The first metric we use is the proportion of solutions (
% $P_{opt}$ from an algorithm $A$ that have the optimal (i.e., minimum possible) length.
% We compute $P_{opt}$ for $A$ by comparing the lengths of the schedules produced by $A$ with those provided by IP.
% Next, we introduce the relative performance $\rho$ of $A$, defined as the length of the schedule obtained from $A$ divided by the optimal length.
% Finally, we compare the algorithms based on the average (per instance) and total time $T$ required by the algorithms.

\end{document}